\documentclass[10pt,journal,cspaper,compsoc]{IEEEtran}

\ifCLASSOPTIONcompsoc
   \usepackage[nocompress]{cite}
\else
\fi

\usepackage{amsmath}
\usepackage{amsthm}
\usepackage{amssymb}
\usepackage{algorithm, algorithmic}
\usepackage{subfigure}
\usepackage{graphicx}
\usepackage{color}
\usepackage{epstopdf}
\usepackage{units}

\newcommand{\half}{{\nicefrac{1}{2}}}
\newcommand{\R}{\mathbb{R}}

\newcommand{\by}{\times}

\newcommand{\norm}[2]{\lVert{#1}\rVert_{#2}}
\newcommand{\ind}[1]{[\![{#1}]\!]}
\newcommand{\diam}{\textrm{diam}}

\newcommand{\vect}[1]{\mathbf{#1}}
\newcommand{\vzero}{\vect{0}}

\newcommand{\vb}{\vect{b}}

\newcommand{\vn}{\vect{n}}
\newcommand{\vp}{\vect{p}}
\newcommand{\vq}{\vect{q}}
\newcommand{\vr}{\vect{r}}

\newcommand{\vv}{\vect{v}}
\newcommand{\vw}{\vect{w}}
\newcommand{\vx}{\vect{x}} 
\newcommand{\vy}{\vect{y}}
\newcommand{\vz}{\vect{z}}

\newcommand{\vt}{\vect{\boldsymbol{\theta}}}
\newcommand{\vmu}{\vect{\boldsymbol{\mu}}}
\newcommand{\vrho}{\vect{\boldsymbol{\rho}}}

\newcommand{\ncw}{p}      
\newcommand{\dd}{n}	  
\newcommand{\Dict}{B}
\newcommand{\CP}{\mathfrak{B}}  
\newcommand{\at}{\vb}
\newcommand{\atm}{\at_{\max}}
\newcommand{\tv}{\vy}

\newcommand{\wv}{\vw}
\newcommand{\we}{w}

\newcommand{\wvo}{\hat{\wv}}
\newcommand{\weo}{\hat{\we}}

\newcommand{\cp}{P} 
\newcommand{\ch}{H} 

\newcommand{\lm}{{\lambda_{\max}}}

\newcommand{\vto}{\hat{\vt}} 
\newcommand{\vtf}{{\color{black}{\vt_{\!F}}}} 
\newcommand{\vzo}{\hat{\vz}} 
\newcommand{\rh}{\hat{r}} 
\newcommand{\sh}{\hat{s}} 

\newcommand{\FS}{\mathcal F} 
\newcommand{\RR}{\mathcal{R}} 
\newcommand{\GR}{\mathcal{G}} 
\newcommand{\DM}{{\color{black}D}} 
\newcommand{\ADM}{{\color{black}\mathcal{R}}} 
\newcommand{\SPH}{{\color{black}S}} 

\newcommand{\vqb}{\vq_d} 
\newcommand{\rd}{r_d} 
\newcommand{\fr}{{\color{black}{\psi}}} 

\newcommand{\ddf}{H}

\newcommand{\DT}{{\rm DT}}
\newcommand{\ST}{{\rm ST}}

\newcommand{\nc}{\tau} 

\newcommand{\lat}{\lambda_{t}}

\newcommand{\IS}{{\cal I}} 
\newcommand{\upa}[1] {{\scriptstyle\uparrow}\hspace{-1.8pt}{\scriptscriptstyle #1}}      
\newcommand{\dna}[1] {{\scriptstyle\downarrow}\hspace{-0.8pt}{\scriptscriptstyle #1}} 

\newcommand{\tu}{V_u}
\newcommand{\tl}{V_l}
\newcommand{\maxu}{M}

\newcommand{\comp}{\preceq}

\newcommand{\disjct}{\vee}

\newcommand{\leqn}{\!\leq\!}
\newcommand{\lessn}{\!<\!}
\newcommand{\geqn}{\!\geq\!}

\newcommand{\fracn}[2]{{#1}/{#2}}

\newtheorem{proposition}{Proposition}
\newtheorem{theorem}{Theorem}
\newtheorem{lemma}{Lemma}
\newtheorem{corollary}{Corollary}

\DeclareMathOperator*{\argmax}{arg\,max}

\DeclareMathOperator*{\sgn}{sign}

\newcommand{\false}{\textsc{\it false}}
\newcommand{\true}{\textsc{\it true}}

\begin{document}

\title{Screening Tests for Lasso Problems}
%
%
%

\author{Zhen~James~Xiang,~\IEEEmembership{}
Yun~Wang,~\IEEEmembership{}
and~Peter~J.~Ramadge~\IEEEmembership{}%
\IEEEcompsocitemizethanks{\IEEEcompsocthanksitem At the time of writing, all authors were affiliated with the Department of Electrical Engineering, Princeton University, Princeton, NJ 08544, USA.\protect\\
E-mail:
\{zhenfavor, ywang721\}@gmail.com,
ramadge@princeton.edu\hfill}
\thanks{}}

%
%

\markboth{IEEE Transactions Journal,~Vol.~X, No.~Y, month~2016}%
{Shell \MakeLowercase{\textit{et al.}}: Accepted paper}
%


\IEEEcompsoctitleabstractindextext{%
\begin{abstract}
This paper is a survey of dictionary screening for the lasso problem.
The lasso problem seeks a sparse linear combination of the columns of a dictionary to best match a given target vector. This sparse representation has proven useful in a variety of subsequent processing and decision tasks.
For a given target vector, dictionary screening quickly identifies a subset of dictionary columns that will receive zero weight in a solution of the corresponding lasso problem.
These columns can be removed from the dictionary prior to solving the lasso problem without impacting the optimality of the solution obtained.
This has two potential advantages:
it reduces the size of the dictionary, allowing the lasso problem to be solved with less resources, and it may speed up obtaining a solution.
Using a geometrically intuitive framework, we provide basic insights for understanding useful lasso screening tests and their limitations.
We also provide illustrative numerical studies on several datasets.
\end{abstract}

}

\maketitle

\IEEEdisplaynotcompsoctitleabstractindextext

%
\IEEEpeerreviewmaketitle

\section{Introduction}\label{sec:introduction}
The sparse representation of data with respect to a dictionary of features
has recently contributed to successful new methods in machine learning, pattern analysis, and signal/image processing.
At the heart of many sparse representation methods is the least squares problem with 
$\ell_{1}$ regularization, often called the lasso problem \cite{RobTib1996}:
\begin{equation} \label{eq:iterW}
	\min_{\wv \in \R^{\ncw}}
	\qquad 	
	\half \norm{\tv- \Dict \wv}{2}^2
	+ \lambda  \|\wv\|_{1},
\end{equation}
where $\lambda >0$ is a regularization parameter.
The matrix $\Dict\in\R^{\dd\times\ncw}$ is called the \emph{dictionary} and its columns $\{\at_i\}_{i=1}^{\ncw}$ are usually called \emph{features}.
Depending on the field, the terms \emph{codewords},
\emph{atoms}, \emph{filters}, and \emph{regressors}
are also used.
The lasso problem seeks a representation of the \emph{target vector}
$\tv\in\R^{\dd}$  as a linear combination $\sum_{i=1}^{\ncw} \we_{i}\at_{i}$
of the features with many $\we_i=0$ (sparse representation).
Equation \eqref{eq:iterW} also serves as the Lagrangian for the widely used constrained problems
$	\min_{\wv\in\R^{\ncw}} \
\norm{\tv-\Dict\wv }{2}^2$
subject to $\|\wv\|_{1} \leq \sigma$,	
and
$	\min_{\wv\in\R^{\ncw}} \ \|\wv\|_{1}$
subject to 	 $\norm{\tv- \Dict \wv}{2}^2  \leq \varepsilon$.
Many solvers of these problems
address the Lagrangian formulation \eqref{eq:iterW} directly \cite{Yang2010A-review}.

The above problems are studied extensively in the signal processing, computer vision, machine learning, and statistics literature.
See, for example, the general introduction to sparse dictionary representation methods in \cite{Elad2010Sparse} and \cite{Wright2010Sparse}.
Sparse representation has proven effective in applications ranging from image restoration \cite{Mairal2008Sparse,Mairal2009Non-local},
to face recognition \cite{Wright2009Robust,Wagner2011Towards},
object recognition \cite{Yu2009Nonlinear},
speech classification  \cite{Sainath2010Bayesian},
speech recognition \cite{Sainath2010Sparse},
music genre classification \cite{Chang2010Music},
and topic detection in text documents \cite{Prasad2011Emerging}.
In these applications, it is common to encounter a large dictionary (e.g.,~in face recognition),
data with large data dimension (e.g.,~in topic detection), and in dictionary learning, a large number of dictionary
iterations (e.g.,~in image restoration).
These factors can make solving problem \eqref{eq:iterW} a bottleneck in the computation.

Several approaches have been suggested for addressing this computational challenge.
In the context of classification,  Zhang et al.\cite{Zhang2011Collaborative} propose abandoning sparsity and using a fast collaborative linear representation scheme based on $\ell_2$ regularized least squares. This improves the speed of classification in face recognition applications. However, 
in general the (nonlinear) Sparse Representation Classifier (SRC) \cite{Wright2009Robust} achieves superior classification accuracy.
Another approach is to seek a sparse representation
using a fast greedy method to approximate the solution of \eqref{eq:iterW}.
There has been a considerable amount of work in this direction, see for example  \cite{Elad2010Sparse, Efron2004LARS,Trop2007OMP}.
However, this approach seems best when seeking very sparse solutions
and, in general, the solutions obtained can be challenging to analyze.

Recently an approach known as (dictionary) screening has been proposed.
For a given target vector $\tv$ and regularization parameter $\lambda$, screening quickly identifies a subset of features that is guaranteed to have zero weight in a solution $\wvo$ of \eqref{eq:iterW}.
These features can be removed (or ``rejected'') from the dictionary to form a smaller, more readily solved lasso problem.
By  padding its solution appropriately with zeros, one
obtains a solution of the original problem.
This approach is the focus of the paper.

Screening has two potential benefits.
First, it can be run in an on-line mode with very few features loaded into memory at a given time. By this means, screening can significantly reduce the size of the dictionary that needs to be loaded into memory in order to solve the lasso problem.
Second,
by quickly reducing the number of features we can often solve problems faster.  Even small gains can become very significant when many lasso problems must be solved.
Moreover, since screening is transparent to the lasso solver, it can be used in conjunction with many existing solvers.

The idea of screening can be traced back to various feature selection heuristics in which selected features $\{\at_i\}$ are used to fit a response vector $\tv$.
This is usually done by selecting features
based on an empirical measure of relevance to $\tv$,  such as the correlation of $\tv$ and $\at_i$.
This is used, for example, in univariate voxel selection based on t-statistics in the fMRI literature \cite{Smith2004Overview}.
Fan and Lv \cite{Fan2008Sure} give an excellent review of recent results on correlation based feature selection and formalize the approach in a probabilistic setting as a correlation based algorithm called Sure Independence Screening (SIS).
In a similar spirit, Tibshirani et al. \cite{Tibshirani2010Strong} report Strong Rules for screening the lasso, the elastic net and logistic regression. These rules are also based on thresholding correlations.
With small probability, SIS and the Strong Rules can yield ``false'' rejections.

A second approach to screening seeks to remove dictionary columns while avoiding any false rejections.
In spirit, this harks back to the problem of removing ``non-binding'' constraints in linear programs \cite{GLThompsonetal1966}.
For the lasso problem, the first line of recent work in this direction is due to El Ghaoui et al.\cite{Ghaoui2012},
where such screening tests are called  ``SAFE'' tests.
In addition to the lasso, this work examined screening for a variety of related sparse regularization problems.
Recent work (e.g., \cite{Xiang2011Learning_b,Xiang2012Fast,
ellpscr2012,dpp2015}) has focused mainly on the lasso problem and close variants.

The basic approach in the above papers
is to bound the solution of the dual problem of \eqref{eq:iterW} within a compact region $\RR$ and find
$\mu_{\RR}(\at)=\max_{\vt\in \RR} \vt^T\at$.
For simple regions $\RR$, $\mu_{\RR}$ is readily computed and yields a screening test for removing a subset of unneeded features.
This approach has resulted in tests based on spherical bounds
\cite{Ghaoui2012, Xiang2011Learning_b},
the intersection of spheres and half spaces (domes)
\cite{Ghaoui2012,Xiang2012Fast},
elliptical bounds \cite{ellpscr2012}
and novel approaches for selecting the parameters of these regions to best bound the dual solution of \eqref{eq:iterW} \cite{dpp2015}.
These screening tests can execute quickly, either serially or in parallel, and require very few features to be loaded into memory at once. If one seeks a strongly to moderately sparse solution,  the tests can significantly reduce dictionary size and speed up the solution of lasso problems.

To keep our survey focused, we concentrate on screening
for the lasso problem.
However, the methods discussed apply to any problem that can be efficiently transformed into a lasso problem.
For example, the elastic net \cite{Zou2005} and full rank generalized lasso problems \cite{TT2011}.
Moreover, the basic ideas and methods discussed are a good foundation for applying screening to other sparse regularization problems.
We will situate our exposition within the context of prior work as the development proceeds.

The main features of our survey include:

\noindent
(a) Our exposition uses a geometric framework which  unifies many lasso screening tests and provides basic tools and geometric insights useful for developing new tests. In particular, we emphasize the separation of the structure or ``architecture'' of the test from the design problem of selecting its parameters.

\noindent
(b) We examine whether more complex screening tests  are worthwhile.
For each $m\geq 0$, there is a family of tests based on the intersection of a spherical bound and $m$ half spaces.
As $m$ increases these tests can reject more features but are also more time consuming to execute. To examine if more complex tests are worthwhile, we derive the region screening test for the intersection of a sphere and two half spaces, and use this to examine where current region screening tests stand in the trade-off between rejection rate and computational efficiency.

\noindent
(c) We show how composite tests can be formed from existing tests.  In particular, we describe a composite test based on carefully  selected dome regions that
performs competitively in numerical studies.
We also point out a fundamental limitation of this approach.

\noindent
(d) We review sequential screening schemes that make headway on the problem of screening for small normalized values of $\lambda$.
When used in an ``on-line'' mode with realistic values of  the regularization parameter, these methods can successfully reduce the size of large dictionaries to a manageable size, allowing larger problems to be solved, and can result in a faster overall computation.

\subsection{Outline of the Paper}
We begin in \S\ref{sec:prelim} with a review of basic tools, especially the dual of the lasso problem and its geometric interpretation.
\S\ref{sec:screen} introduces screening in greater detail and \S\ref{sec:region} introduces region tests.
After these preparations, we discuss several important forms of region tests:
sphere tests (\S\ref{sec:st}), sphere plus hyperplane tests (\S\ref{sec:dt})
and sphere plus two hyperplane tests (\S\ref{sec:THT}).
We show how spherical bounds can be iteratively refined using features (\S\ref{sec:refine}), and examine ways to combine basic tests.
\S\ref{sec:sarp} gives a brief overview of sequential screening.
We give a practical summary of screening algorithms in \S\ref{sec:alg}
and illustrate the results of screening via numerical studies in \S\ref{sec:exp}.
We conclude in \S\ref{sec:discuss}.
Proofs of new or key results are given in the Appendices, organized by the section in which the result is discussed.

\section{Preliminaries}\label{sec:prelim}
We focus on the lasso \eqref{eq:iterW}, but
it will be convenient to also consider the
\emph{nonnegative} lasso:
 \begin{equation}
	\label{eq:iterW_nn}
	\begin{split}
	\min_{\wv\in\R^{\ncw}} & \qquad \half
	\norm{\tv- \Dict \wv}{2}^2
	+ \lambda \|\wv\|_{1}, \\
	\text{s.t.} \quad & \qquad \wv \geq 0. \\
	\end{split}
\end{equation}
The analysis and algorithms in the paper apply (with minor changes) to both problems.

Throughout the paper we assume that a fixed dictionary $\Dict$ is used to solve various instances of \eqref{eq:iterW} or \eqref{eq:iterW_nn}. We assume that all features are nonzero and say that the dictionary is \emph{normalized} if all features
have unit norm.
Each instance is specified by a pair $(\tv,\lambda)$ consisting of a target vector $\tv$ and a value $\lambda$ of the regularization parameter.

Multiplying the objective of \eqref{eq:iterW} by $\alpha^2$, with $\alpha>0$,
yields the equivalent problem:
$$
	\min_{\wv \in \R^{\ncw}} \
	\half \norm{ \bar{\tv}- \bar{\Dict} \wv}{2}^2
	+ \bar \lambda  \|\wv\|_{1},
$$
where $\bar{\tv}=\alpha \tv$, $\bar{\Dict}=\alpha \Dict$, and $\bar{\lambda}=\alpha^2 \lambda$.
Some lasso solvers require that $\|\Dict\|_F\leq 1$,
and problem \eqref{eq:iterW} must be scaled
to ensure this holds.
As a result, it is meaningless to talk about the value of $\lambda$ employed when solving \eqref{eq:iterW} without accounting for possible scaling. One way to
do this is to define $\lm = \max_{j=1}^{\ncw} |\at_j^T \tv|$. Then the ratio $\lambda/\lm$ is invariant to scaling. The parameter $\lm$ is also useful for other purposes related to screening. Throughout the paper, we use the ratio $\lambda/\lm$ as an unambiguous measure of the amount of regularization used in solving \eqref{eq:iterW} and \eqref{eq:iterW_nn}.

Geometric insight on lasso problems, and on screening in particular, is enhanced by bringing in the Lagrangian dual of \eqref{eq:iterW}.
The following parameterization of the dual problem
is particularly convenient \cite{Xiang2011Learning_b}:
\begin{equation}
	\begin{split}
		\label{eq:dual}
		\max_{\vt\in \R^{\dd}}  \qquad &
		\half \norm{\tv}{2}^2-
		\nicefrac{\lambda^2}{2}\norm{\vt -\tv/\lambda}{2}^2\\
		\text{s.t.} \qquad & |\vt^T\at_i| \leq 1 \quad \forall
		i=1,2,\ldots,\ncw.
	\end{split}
\end{equation}
Solutions $\wvo\in\R^{\ncw}$ of \eqref{eq:iterW} and
$\vto\in \R^{\dd}$ of \eqref{eq:dual} satisfy:
\begin{equation}
	\label{eq:relationship}
	\tv = \Dict \wvo + \lambda \vto,
		\ \  \vto^T\at_i =
		\begin{cases}
		  \sgn{\weo_i}, &\text{ if } \weo_i\neq 0;\\
		  \gamma\in \left[-1,1\right], & \text{ if } \weo_i=0.
		\end{cases}
\end{equation}
The corresponding dual problem of \eqref{eq:iterW_nn} is:
\begin{equation}
	\begin{split}
		\label{eq:dual_nn}
		\max_{\vt\in \R^{\dd}}  \qquad
		& \half \norm{\tv}{2}^2 - \nicefrac{\lambda^2}{ 2}\norm{\vt -\tv/\lambda}{2}^2\\
		\text{s.t.} \qquad & \vt^T\vb_i \leq 1 \quad \forall
		i=1,2,\ldots,\ncw,
	\end{split}
\end{equation}
with the primal and dual solutions related via:
\begin{equation} \label{eq:relationship_nn}
	\tv = \Dict \wvo + \lambda \vto,
		\  \vto^T\at_i =
		\begin{cases}
		  1, & \text{ if } \weo_i>0; \\
		  \gamma \in\left(-\infty,1\right],
		  &\text{ if } \weo_i=0.
		\end{cases}
\end{equation}
A derivation of \eqref{eq:dual} and \eqref{eq:relationship} is given in the Appendix.
It will be convenient to define a \emph{feature pool} $\CP$.
For the lasso, $\CP=\{\pm\at_i\}_{i=1}^{\ncw}$
and for the nonnegative lasso, $\CP=\{\at_i\}_{i=1}^{\ncw}$.
This allows the constraints in \eqref{eq:dual} and \eqref{eq:dual_nn}
to be stated as $\forall \at\in\CP: \vt^T\at\leq 1$.

For $\vx\in \R^{n}$, let $\cp(\vx)=\{\vz:\vx^T\vz=1\}$ denote the hyperplane in $\R^{n}$
that has unit normal $\vx/\|\vx\|_2$ and contains
the point $\vx/\|\vx\|_{2}^{2}$.
Let $\ch(\vx)=\{\vz:\vx^T\vz\leq 1\}$
denote the corresponding closed half space containing the origin. So a constraint of the form $\at^T \vt\leq 1$ requires that $\vt$ lies in the closed half space $\ch(\at)$.
Hence the set of feasible points $\FS$ of the dual problems is the nonempty, closed, convex set formed by the intersection of the finite set of closed half spaces $\ch(\at)$, $\at\in\CP$.
This is illustrated in Fig.~\ref{fig:geometrydp} and~\ref{fig:geometryndp}.
In addition, for the lasso, $\vt \in \FS$ if and only if  $-\vt\in \FS$. So $-\FS=\FS$.
This follows from the same property  of the feature pool: $-\CP=\CP$.

\begin{figure*}[t]
\centering
	\subfigure[]{
	\includegraphics[width=0.245\linewidth]
	{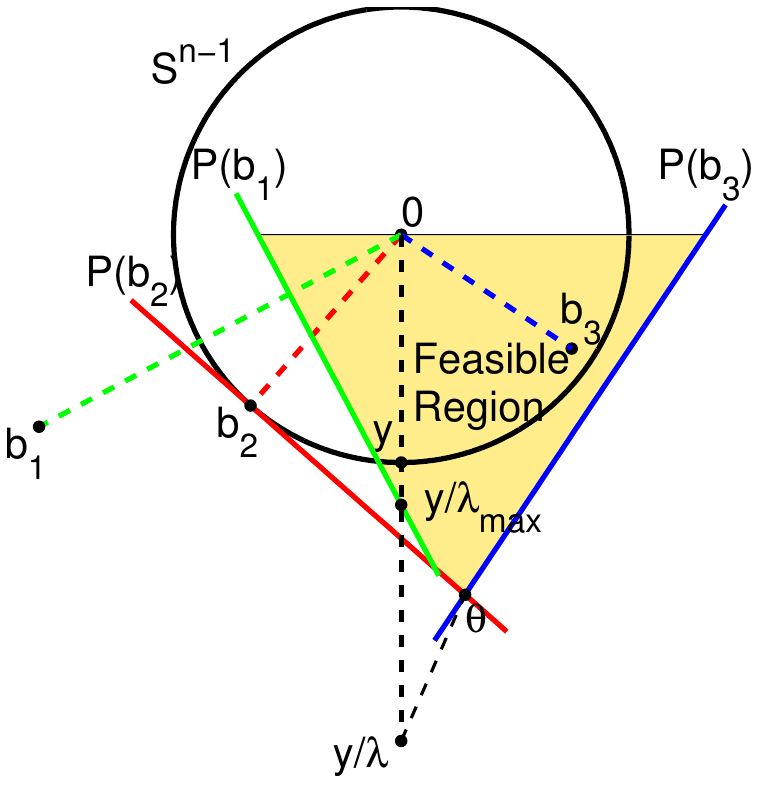}
	\label{fig:geometrydp}
	}
	\qquad
	\subfigure[]{
	\includegraphics[width=0.22\linewidth]	
	{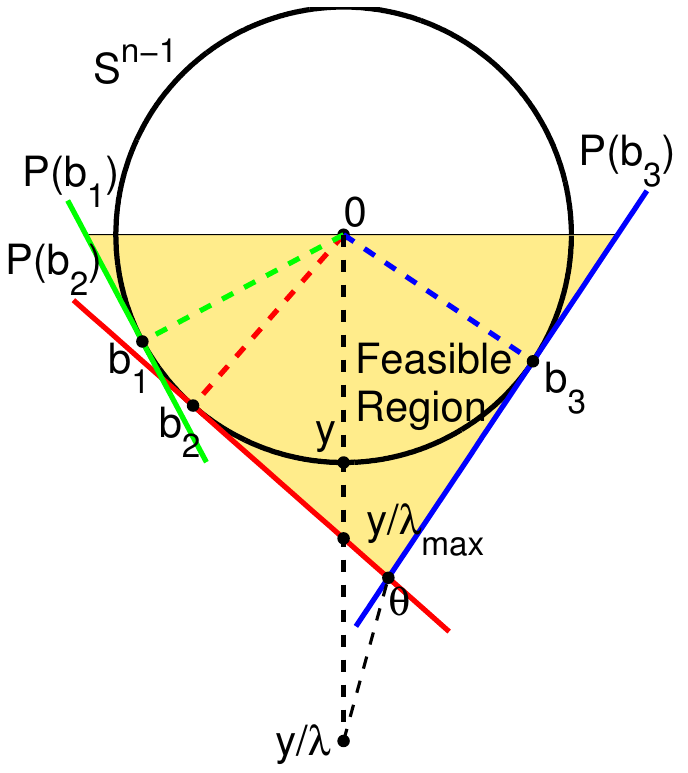}\label{fig:geometryndp.pdf}
	\label{fig:geometryndp}
}
\qquad
	\subfigure[]{\includegraphics[width=0.22\linewidth]	
	{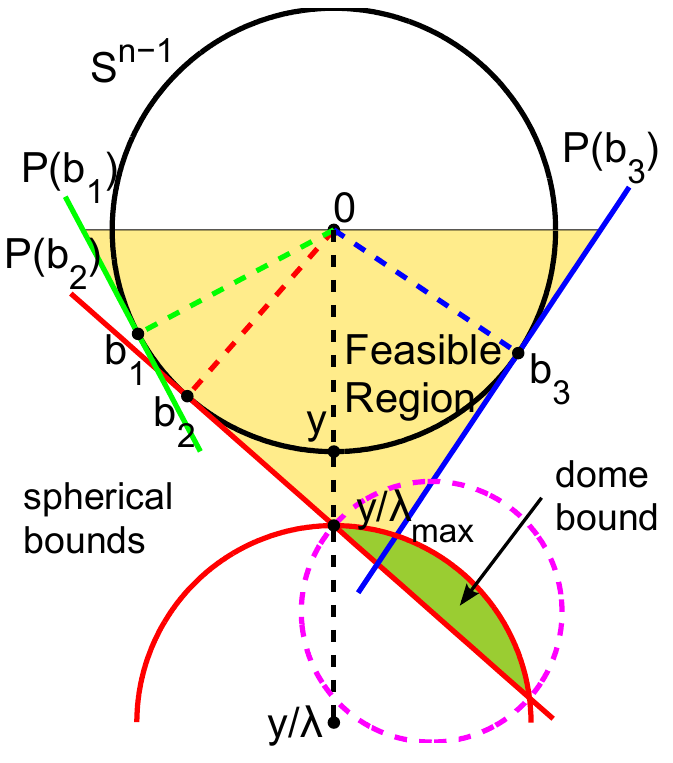}\label{fig:3ST}}
	\label{fig:geometryndp2s}
\caption{\small
The constraints and feasible set $\FS$ of the dual problem for
	{\bf (a):} general features,
	{\bf (b):} unit norm features.
	{\bf (c):} Examples of two spheres and a dome
	region bounding $\vto$ for unit norm features.
	In all cases only the lower half of $\FS$ is shown.
}
\label{fig:prelim}
\end{figure*}

To maximize the objective function in \eqref{eq:dual} or \eqref{eq:dual_nn} we seek
the projection $\vto$ of $\tv/\lambda$ onto the
closed convex set $\FS$.
This is the unique point satisfying
the following set of inequalities \cite[\S3.1]{HU2001}:
for each $\vt\in \FS$,
\begin{equation}\label{eq:vtoineq}
(\tv/\lambda-\vto)^{T}(\vt-\vto) \leq 0.
\end{equation}
In contrast, the lasso problem \eqref{eq:iterW} may not have a unique solution
\cite{Fuchs2005,RJT2012}.

The set of points $\{\vto(\lambda), \lambda>0\}$ is called the
\emph{dual regularization path}. For $\lambda$ sufficiently large,
$\tv/\lambda$ lies in $\FS$ and $\vto(\lambda)=\tv/\lambda$.
To find the smallest $\lambda$ for which this holds, let
\begin{align}
\lm &= \max_{\at\in\CP} \tv^T\at,\\
\atm &\in \argmax_{\at\in\CP} \tv^T\at.
\end{align}
Then  for all $\at\in\CP$:
$(\fracn{\tv}{\lm})^{T}\at \leq  \fracn{\tv^T\atm}{\lm} = 1$.
So $\tv/\lm$ lies in the boundary of $\FS$.
As $\lambda\geq \lm $ decreases from a large value, $\vto(\lambda)=\tv/\lambda$ moves in a straight
line within $\FS$ until $\lambda=\lm$, at which point $\vto(\lm)=\tv/\lm$ first lies on the boundary of $\FS$.
As $\lambda$ decreases below $\lm$,
$\tv/\lambda$ moves away from $\FS$ and $\vto(\lambda)$
is the unique projection of $\tv/\lambda$
onto the boundary of $\FS$.
Using \eqref{eq:relationship},
for $\lambda/\lm>1$, $\wvo=0$, and conversely, if $\wvo=0$, then $\vto=\tv/\lambda \in \FS$.
So for $\lambda/\lm \in (0,1)$, $\tv/\lambda \notin \FS$, $\vto(\lambda)$ lines on the boundary of $\FS$,
and $\wvo$ is nonzero.

Let $\IS=[1,\dots,\ncw]$ denote the ordered set of feature indices and $S\subset \IS$.
Given $\wv\in \R^{\ncw}$, let $\wv_{\dna{S}}$ denote the vector in $\R^{|S|}$ obtained by subsampling $\wv$ at the indices in $S$.
Conversely, for $\vz\in \R^{|S|}$, let $\vz^{\upa{S}}$ denote the vector in $\R^{\ncw}$ obtained by upsampling $\vz$: the entries of
$\vz^{\upa{S}}$ with indices in $S$ take the corresponding values in
$\vz$ and all other entries are zero. Similarly, for a dictionary
$\Dict\in \R^{\dd\times \ncw}$, let
$\Dict_{\dna{S}}$ denote the subdictionary
obtained by sampling the columns of $\Dict$
at the indices in $S$.
The following properties are clear:
(a) $\vz=(\vz^{\upa{S}})_{\dna{S}}$;
(b) if $\we_{i}=0$ for $i\notin S$, then $\wv=(\wv_{\dna{S}})^{\upa{S}}$;
(c) $\Dict_{\dna{S}}\vz=\Dict\vz^{\upa{S}}$; and
(d)  if $\we_{i}=0$ for $i\notin S$, $\Dict \wv=\Dict_{\dna{S}}\wv_{\dna{S}}$.

By \eqref{eq:relationship}, if we know the primal solution $\wvo$, then the dual solution is
$\vto=(\tv-\Dict \wvo)/\lambda$. Conversely, if we know the dual solution $\vto$, then any point satisfying the following equations is a primal solution: 
\begin{align}
\begin{split} \label{eq:primalfromdual}
\Dict_{\dna A(\vto)} \vw_{\dna A(\vto)}
& =  \tv - \lambda \vto \\
\vw_{\dna A(\vto),i} (\vto^T\at_i)
	&\geq 0, \ i\in A(\vto)
\end{split}
\end{align}
where $A(\vto)=\{i\colon |\vto^{T}\at_{i}|=1\}$.

\section{Screening}\label{sec:screen}
We now explain the idea of screening in detail.
Given an instance $(\tv,\lambda)$  of \eqref{eq:iterW},
we select a partition $\IS=S\cup \bar S$ of the features.
We say that the features indexed by $S$ are
{\em selected} and those indexed by $\bar S$ are {\em rejected}.
Then we form the reduced dictionary $B_{\dna{S}}$ of selected features and let
$\vzo$ denote a solution of the corresponding lasso problem using this dictionary.
In general, the upsampled vector $\vzo^{\upa{S}}$ is not a solution of the original lasso problem  \eqref{eq:iterW}.
Here is the key point: screening seeks a partition such that the upsampled vector $\vzo^{\upa{S}}$ solves \eqref{eq:iterW}.
In general, such a partition depends on the instance and hence must be computed ``on-the-fly''.

By virtue of being smaller, the reduced problem is more manageable.
For example, it may fit into memory when
the original problem does not, and finding
its solution may require less time.
Hence there are two evaluation metrics of interest:
the size of $S$ (or $\bar S$) as a fraction of $\IS$,
and the total time taken to select $S$ \emph{and} solve the reduced problem relative to the time taken to solve the original problem directly without screening.
We will normally express these metrics as the \emph{rejection fraction} $|\bar S|/|\IS|$
and the \emph{speedup factor} $t_{\textrm{solve}}/(t_{\textrm{screen}}+t^{\textrm{r}}_{\textrm{solve}})$.
Here $t_{\textrm{solve}}$ is the time to solve the original lasso problem,
$t_{\textrm{screen}}$ is the time to select the partition (screen the dictionary), and
$t^{\textrm{r}}_{\textrm{solve}}$ is the time to solve the reduced lasso problem.

Not surprisingly, if we know the dual solution $\vto$, then it is easy to come up with a suitable partition. To see this, consider the lasso  problem.
For any partition $S \cup \bar S$, let $\wvo$ and $\vzo$ denote solutions of the
original and reduced lasso problems, respectively.
It is clear that the following always holds:
\begin{align}
\begin{split}\label{eq:partprop}
&\half \|\tv-\Dict\wvo \|_{2}^{2} +\lambda \|\wvo\|_{1} \\
&\leq  \half \|\tv-\Dict \vzo^{\upa{S}} \|_{2}^{2} +\lambda \|\vzo^{\upa{S}}\|_{1}\\
&= \half \|\tv-\Dict_{\dna{S}}\vzo \|_{2}^{2} +\lambda \|\vzo\|_{1}\\
&\leq \half \|\tv-\Dict_{\dna{S}}\wvo_{\dna{S}} \|_{2}^{2}
	+\lambda \|\wvo_{\dna{S}}\|_{1} .
\end{split}
\end{align}
Now assume the dual solution $\vto$ is known,
let $A(\vto)=\{i\colon |\vto^{T}\at_{i}|=1\}$ denote the active constraints at $\vto$,
and consider the particular partition  $A(\vto)\cup \bar A(\vto)$.
Equation \eqref{eq:relationship} shows that if
$|\vto^T\at_i| < 1$ (equivalently, $i\in \bar A(\vto)$), then $\weo_i=0$.
Hence for this partition,  $\Dict \wvo=\Dict_{\dna{A(\vto)}} \wvo_{\dna{A(\vto)}}$,
$\|\wvo_{\dna{A(\vto)}}\|_{1} =\|\wvo\|_{1}$, and
\begin{align}
\begin{split}\label{eq:dualsolnpartition}
&\half \|\tv-\Dict\wvo \|_{2}^{2} +\lambda \|\wvo\|_{1}\\
& = \half \|\tv-\Dict_{\dna{A(\vto)}}\wvo_{\dna{A(\vto)}} \|_{2}^{2}
	+\lambda \|\wvo_{\dna{A(\vto)}}\|_{1} .
\end{split}
\end{align}
Equation \eqref{eq:dualsolnpartition} implies that for this partition the two inequalities in \eqref{eq:partprop} must be equalities. It follows that $\wvo_{\dna{A(\vto)}}$ solves the
reduced problem and $\vzo^{\upa{S}}$ solves the original problem.
Although a simple observation, this is worth stating as a theorem.

\begin{theorem}\label{thm:crt}
Let the solution $\vto$  of \eqref{eq:dual} (resp. \eqref{eq:dual_nn}) have active set $A(\vto)$.
If $\vzo$ is a solution of \eqref{eq:iterW}
(resp. \eqref{eq:iterW_nn})
with dictionary $\Dict_{\dna{A(\vto)}}$, then $\vzo^{\upa{A(\vto)}}$ solves  \eqref{eq:iterW} (resp. \eqref{eq:iterW_nn}).
Moreover, every solution of \eqref{eq:iterW} (resp. \eqref{eq:iterW_nn})
can be expressed in this way.
\end{theorem}

The fundamental partition of $\IS$ into $A(\vto)$ and $\bar A(\vto)$
is conceptually very important but  obviously impractical.
If we know $\vto$, then we can easily solve the primal problem
(see \eqref{eq:primalfromdual})
and this makes screening  and problem reduction unnecessary.
As a first step towards finding a practical way to partition the features,
we note that  if $A(\vto)\subseteq S$ (screening keeps more) or equivalently $\bar S \subseteq \bar A(\vto)$ (screening rejects less), then equation \eqref{eq:dualsolnpartition} holds with $S$ replacing $A(\vto)$. This implies that the two inequalities in  \eqref{eq:partprop} hold with equality for this partition. Hence we have the following corollary of Theorem \ref{thm:crt}.

\begin{corollary}\label{cor:compliantS}
Let the solution $\vto$ of \eqref{eq:dual} (resp. \eqref{eq:dual_nn}) have active set
$A(\vto)$. Let $A(\vto)\subseteq S \subseteq \IS$.
If $\vzo$ is a solution of \eqref{eq:iterW}
(resp. \eqref{eq:iterW_nn})
with dictionary $\Dict_{\dna{S}}$,
then $\vzo^{\upa{S}}$ is a solution of \eqref{eq:iterW} (resp. \eqref{eq:iterW_nn}).
Moreover, every solution of \eqref{eq:iterW} (resp. \eqref{eq:iterW_nn})
can be expressed in this way.
\end{corollary}

\section{Region Tests}\label{sec:region}
The core idea  for creating a partition of the dictionary that conforms with Corollary \ref{cor:compliantS} is to bound $\vto$ within a compact region $\RR$.
For each feature $\at$, we then compute $\mu_{\RR}(\at)=\max_{\vt\in \RR} \vt^{T}\at$, and use this quantity to partition $\Dict$ \cite{Ghaoui2012}.

We first illustrate this for the nonnegative lasso.
For a compact set $\RR$, if $\RR=\emptyset$, all features are rejected;
otherwise for each feature $\at_i$,
$\mu_{\RR}(\at_i)=\max_{\vt\in\RR} \vt^T\at_i$
exists. Then define the partition:
\begin{align}
	\at_{i} \in  \begin{cases}
	\bar S, & \textrm{if } \mu_{\RR}(\at_i) < 1; \\
	S, & \textrm{otherwise.}
	\end{cases}
	\label{eq:rrpnn}
\end{align}
The logic is that if $\vto\in \RR$ and $\mu_{\RR}(\at_i)<1$,
then $\vto^T\at_i<1$ and hence $i\in \bar {A}(\vto)$.
Thus $\bar S\subseteq \bar{A}(\vto)$, as desired.

For the lasso problem, $\vto\in \RR$ and
$\mu_{\RR}(\at_i)<1$ ensure $\vto^T\at_i<1$.
But in this case we also need
$-1 < \vto^T\at_i$ or equivalently
$\vto^T(-\at_i) <1$.
This holds if $\mu_{\RR}(-\at_i)<1$.
Effectively, we must test both $\at_{i}$ and $-\at_{i}$ to account for the
positive or negative sign of $\we_{i}$.
So for the lasso the partition is:
\begin{align}
	\at_{i} \in \begin{cases}
	\bar S, & \textrm{if } \max\{\mu_{\RR}(\at_i), \mu_{\RR}(-\at_i)\}<1 ;\\
	S, & \textrm{otherwise.}
	\end{cases}  \label{eq:rrp}
\end{align}
For example, when $\RR=\{\vto\}$, $i\in \bar S$ if:
(a) $\vto^{T}\at_{i}<1$ (nonnegative lasso) and
(b) $|\vto^{T}\at_{i}|<1$ (lasso).
So $\RR=\{\vto\}$, yields the ideal partition $A(\vto)\cup \bar A(\vto)$.

From the above constructions, we see that $\vto\in \RR$ ensures that the partitions \eqref{eq:rrpnn} and \eqref{eq:rrp} satisfy $\bar S \subseteq \bar A(\vto)$. Hence the assumptions of Corollary \ref{cor:compliantS} are satisfied. This is summarized in the following corollary.

\begin{corollary} \label{cor:comregiontest}
Let $\RR$  be a compact region with $\vto\in \RR$.
Then $\RR$ defines a dictionary partition $S\cup \bar S$ with $\bar S \subseteq \bar A(\vto)$.
\end{corollary}

It will be convenient to encode the partition induced by a bounding region $\RR$ as a rejection test $T_{\RR}$ with $T_{\RR}(\at)=1$ if $\at\in \bar S$
and $0$ otherwise.
For example, the rejection test corresponding to \eqref{eq:rrp} is:
\begin{align}
	T_{\RR}(\at_{i}) \in  \begin{cases}
	1, & \textrm{if } \max\{\mu_{\RR}(\at_i), \mu_{\RR}(-\at_i)\}<1; \\
	0, & \textrm{otherwise.}
	\end{cases}
	\label{eq:rrt}
\end{align}

We end this section by noting that for a given dictionary $\Dict$, the partial order of subsets of features induces a partial order on screening tests.
Test $T^{\prime}$ is \emph{weaker} than test $T$,  denoted $T'\comp T$,
if the set of features rejected by $T^{\prime}$ is a subset of the features rejected by $T$.
For example, if $\vto\in \RR$, then $T_{\RR}\comp T_{\{\vto\}}$.
This is a special case of the following lemma.

\begin{lemma}\label{lem:subset}
	If $\RR_1 \subseteq \RR_2$, then $T_{\RR_2} \comp T_{\RR_1}$.
\end{lemma}

If $\RR_1\subset \RR_2$, then the region test for $\RR_1$  
can potentially reject more features than the test for $\RR_2$.

\subsection{The Sphere-Hyperplane Architecture}
We now consider particular forms of bounding regions for $\vto$.
A natural form of bounding region consists of the intersection of a spherical bound with a finite number of half spaces.
The spherical bound arises naturally once we know a dual feasible point, and half spaces arise naturally since these define the dual feasible region $\FS$ (see \eqref{eq:dual}), and are integral to the projection of a point onto $\FS$ (see \eqref{eq:vtoineq}).

The intersection of  a closed ball
$\SPH(\vq,r)=\{\vz: \norm{\vz-\vq}{2}\leq r\}$ with center $\vq$ and radius $r$,
and $m$ half spaces $\vn_i^T\vt \leq c_i$, $i=1,\dots,m$, gives the region:
$$
\RR = \{\vt\colon \|\vt-\vq\|_2 \leq r\}
\  \cap \
\cap_{i=1}^m \{\vt\colon \vn_i^T \vt \leq c_i\} \ .
$$
To form the corresponding region test,
we find $\mu(\at)=\max_{\vt\in \RR} \vt^T\at$ by
solving the optimization problem:
\begin{align}\label{eq:optprobm}
\begin{split}
\min_{\vt}
	& \quad (-\vt^{T} \at) \\
\textrm{s.t.}
	& \quad (\vt-\vq)^{T}(\vt-\vq) -r^{2} \leq 0\\
 	& \quad \vn_i^{T}\vt -c_i \leq 0, \quad i=1,\dots,m.
\end{split}
\end{align}
Once $\mu(\at)$ is known, \eqref{eq:rrt} gives the corresponding screening test.
Using the change of variable $\vz=(\vt-\vq)/r$, problem  \eqref{eq:optprobm} can be simplified to:
\begin{align} \label{eq:optprobms}
\begin{split}
\bar{\mu}(\at)=\min_{\vz }
& \quad (-\vz^T\at) \\
\text{s.t.}
&\quad  \vz^T\vz-1 \leq 0 \\
&\quad  \vn_i^T\vz+\psi_i \leq 0, \quad i = 1,\dots,m.
\end{split}
\end{align}
where $\psi_i = (\vn_i^T\vq-c_i)/r$.
The solution of \eqref{eq:optprobm} is then $\mu(\at)=\vq^T\at+r\bar{\mu}(\at)$.
By decomposing $\vz$ and $\vb$ in terms of $\textrm{span}\{\vn_i\}_{i=1}^m$ and its orthogonal complement, \eqref{eq:optprobms} reduces to a convex program in $\R^{m+1}$.

Increasing $m$ results in tests with the potential to reject more features,
but which are also more complex and time consuming to execute. In the following two subsections, we discuss the simplest cases: $m=0$ (sphere tests), and $m=1$ (dome tests). This gives insight into basic tests and makes connections with the literature.

\subsection{Sphere Tests} \label{sec:st}
Consider bounding $\vto$ within a closed ball
$\SPH(\vq,r)=\{\vz: \norm{\vz-\vq}{2}\leq r\}$ with center $\vq$ and radius $r$.
This bound gives a simple, efficiently implemented test, and it is also a useful building block for more complex tests.
We first determine a close form expression for
$\mu_{\SPH(\vq,r)}(\at)=\max_{\vt\in\SPH(\vq,r) } \vt^T\at$.
An expression for a \emph{sphere test}
$T_{\SPH(\vq,r)}$ then follows from \eqref{eq:rrt}.

\begin{lemma}\label{lem:maxStb}
For $\SPH(\vq,r) = \{\vz: \norm{\vz-\vq}{2}\leq r\}$
and $\at\in \R^{\dd}$:
	\begin{align}	\label{eq:musph}
		\mu_{\SPH(\vq,r)}(\at)
		= \vq^T\at +r\|\at\|_2 .
	\end{align}
\end{lemma}

\begin{theorem}\label{thm:st}
The screening test for the sphere $\SPH(\vq,r)$ is:
\begin{equation}
	T_{\SPH(\vq,r)}(\at)
	\!=\!  \begin{cases}
	1, &\!\!\!\!\textrm{if } \tl(\|\at\|_2) \!<\!
	\vq^T\at  \!<\!  \tu(\|\at\|_2); \\
	0, &\!\!\!\!\textrm{otherwise.}
	\end{cases}  \label{eq:st}
\end{equation}
where
$\tu(t)=1-r t$ and for the lasso
$\tl(t) =-\tu(t)$, and for the nonnegative lasso
$\tl(t)=	-\infty$.
\end{theorem}
For the lasso, the test \eqref{eq:st} can also be written as:
\begin{align}\label{eq:stavf}
	T_{\SPH(\vq,r)}(\at) =
	\begin{cases}
	1, & \textrm{if } |\vq^T\at|
	<  1-r\|\at\|_2; \\
	0, & \textrm{otherwise.}
	\end{cases}
\end{align}

Theorem \ref{thm:st} defines a parametric family of tests:
$\{\ST(\vq,r)\colon \vq\in \R^{\dd}, r\geq 0\}$,
where $\ST(\vq,r)$ denotes the sphere test
with center $\vq$ and radius $r$.
To use a sphere test one first selects values of $\vq$ and $r$ so that $\SPH(\vq,r)$ bounds $\vto$.
We call this the \emph{parameter selection problem}.
By Lemma \ref{lem:subset}, a tighter bound has potential for better screening.
So using only the information provided, and limited computation, we want to select $\vq$ and $r$ to give the ``best bound''.
This is a design problem involving a trade-off between the computation cost to select $\vq$ and $r$ and the resultant screening performance.
Hence we don't expect there is a ``best answer''.
We outline below several selection methods.

\subsubsection{Parameter selection}
If we know a dual feasible point $\vtf\in \FS$,
then $\vto$ can't be further away from $\fracn{\tv}{\lambda}$
than $\vtf$. This gives the basic spherical bound:
\begin{equation}\label{eq:sb0}
\|\vto- \fracn{\tv}{\lambda}\|_{2} \leq \norm{\vtf-\fracn{\tv}{\lambda} } {2},
\end{equation}
with center $\vq=\tv/\lambda$ and radius
$r=\|\vtf-\tv/\lambda\|_{2}$.
In particular, $\vto(\lm)=\tv/\lm$ is dual feasible and
gives a particular instance of \eqref{eq:sb0}:
\begin{equation}\label{eq:sb1}
\|\vto- \tv/\lambda\|_{2} \leq |\fracn{1}{\lambda}-\fracn{1}{\lm}|~\|\tv\|_2.
\end{equation}
This bound is shown in Fig.~\ref{fig:3ST} as the larger sphere in solid red. The bound \eqref{eq:sb1} requires only the specification of the lasso problem and the computation of $\lm$.
We call it the \emph{default spherical bound}.

Better bounds are possible with additional computation
or if additional information is supplied.
For example,
\cite{dpp2015} observed that to obtain a
feasible point $\vtf$ closer to $\vto$ than $\tv/\lm$
one can first run $K$ steps of the homotopy algorithm on \eqref{eq:iterW}.
This gives the solution $\wvo_K$ of the
instance $(\tv,\lambda_K)$, $\lambda_K>\lambda$,
for the $K$-th breakpoint on the (primal) regularization path.
Effectively, this first solves the lasso problem for $\lambda_K>\lambda$,
and then uses this solution to help screen for the actual instance to be solved.
The sphere center can also be moved away from $\tv/\lambda$.
Examples include the sphere tests ST2 and ST3
in \cite{Xiang2011Learning_b} derived in the setting of unit norm $\tv$ and $\at_i$.
In addition, \cite{dpp2015} noted that if the dual solution $\vto_0$ is known for an instance
$(\vx,\lambda_0)$, then
$\|\vto(\lambda)-\vto_0\|_2 \leq  |1/\lambda -1/\lambda_0|~\|\tv\|_2$
(this is discussed further below).
This leverages a solved instance to give a spherical bound centered at
$\vq=\vto_0$.

\subsubsection{Connections with the Literature}
A variety of existing screening tests for the lasso are sphere tests.
The Basic SAFE-LASSO test \cite{Ghaoui2010Safe}
and the test ST1 in \cite[Sect. 2]{Xiang2011Learning_b}
are sphere tests based on the default spherical bound \eqref{eq:sb1}.
The SAFE-LASSO test \cite[Theorem 2]{Ghaoui2010Safe} is also a sphere test.
It assumes a dual feasible point $\vt_{0}$ is given and uses this to
improve the default spherical bound centered at $\tv/\lambda$.
The sphere tests ST2 and ST3 in \cite[Sect. 2]{Xiang2011Learning_b}
use spherical bounds not centered at $\tv/\lambda$.
We will comment further on the test ST3 at the end of \S\ref{sec:comp}.
The core test used in  \cite{dpp2015} is a sphere test with center $\vto_0=\vto(\lambda_0)$, where $\vto_0(\lambda_0)$ is the dual solution at $\lambda_0$, and radius $|1/\lambda-1/\lambda_0|~\|\tv\|_2$.
This bound follows from the nonexpansive property of
projection onto a convex set:
\begin{align}
\begin{split}\label{eq:dppsr}
\|\vto(\lambda)-\vto(\lambda_0)\|_2
&\leq \|\tv/\lambda - \tv/\lambda_0\|_2\\
&= |1/\lambda-1/\lambda_0|~\|\tv\|_2.
\end{split}
\end{align}

The Strong Rule \cite{Tibshirani2010Strong} is also a sphere test for the lasso problem. For notational simplicity, let the features and the target vector $\tv$ have unit norm.
The Strong Rule discards feature $\at_i$ if
$|\at_i^T\tv| <2\lambda-\lm$.
This is a sphere test with center $\vq=\tv/\lambda$ and radius
$r_{sr}=(\lm-\lambda)/\lambda$.  
The point $\vto$ is bounded within the default sphere 
(center $\tv/\lambda$, radius $r=1/\lambda - 1/\lm$).
The Strong Rule uses a sphere with the same center but 
a radius only a fraction of $r$: $r_{sr} = r \lm$.
This smaller sphere is not guaranteed to contain $\vto$.
So the Strong Rule can (with low probability) yield false rejections.
A detailed discussion of this issue is given in \cite{Tibshirani2010Strong}.
A more advanced version of the Strong Rule,
the Strong Sequential Rule \cite{Tibshirani2010Strong},
assumes a solution $\wvo_0$ of the lasso instance $(\tv,\lambda_0)$
is available, where $\lambda_0>\lambda$.
It then forms the residual $\vr_0=\tv-\Dict \wvo_0$ 
and screens the lasso instance $(\tv,\lambda)$ 
using the test $|\at_i^T \vr_0| <2\lambda-\lambda_0$.
This is also a sphere test. To see this, use
\eqref{eq:relationship} to write
$\vr_0=\tv-\Dict \wvo=\lambda_0 \vto_0$.
Then the test becomes
$|\at_i^T \vto_0| <1-r_{ssr}$
with $r_{ssr} = (1/\lambda-1/\lambda_0) 2\lambda$.
This is a sphere test with center $\vto_0$ and
radius $r_{ssr}=2\lambda r$
where $r$ is the radius of the known bounding sphere \eqref{eq:dppsr}.
When $\lambda<0.5$, this test may also yield false rejections. 
See  \cite{Tibshirani2010Strong} for examples and analysis of rule violations.

The SIS test in \cite{Fan2008Sure} is framed in a probabilistic setting and is not intended for lasso screening.
Nevertheless, if we translate SIS into our setting it is a sphere test for a lasso problem with appropriately selected  $\lambda$.
SIS assumes a dictionary $\Dict \in \R^{\dd\by \ncw}$ of standardized vectors (features of unit norm) and  computes the vector of (marginal) correlations $\vrho=\Dict^T \tv$.
Then given $0<\gamma<1$, it selects the top
$[\gamma \dd]$ features ranked by $|\rho_i|$.
Assume for simplicity that the values $|\rho_i|$ are distinct and let $t_{\gamma}$ denote the value of
$|\rho_i|$ for the $[\gamma n]$-th
feature in the ranking.
The SIS rejection criterion can then be written as
$|\at_i^T\tv| < t_{\gamma}$.
We now form a lasso problem with dictionary $\Dict$,
target vector $\tv$, and a value $\lambda/\lm$ to be decided.
For simplicity of notation, assume that $\tv$ has unit norm. Then the default spherical bound for the dual solution of the lasso has center $\tv/\lambda$ and radius $r=1/\lambda-1/\lm$, and the corresponding sphere test is  $|\at_i^T \tv| < \lambda(1-r)$.
Equating the right hand sides of the above test expressions, and using some algebra shows that if we take $\lambda/\lm = \fracn{(1+t_{\gamma})}{(1+\lm)}<1$,
then SIS is the default sphere test for this
particular lasso problem.

\subsection{Sphere Plus Halfspace Tests}\label{sec:dt}
Now consider a region test based on the nonempty intersection of  a spherical ball
$\{\vz \colon  \norm{\vz-\vq}{2}\leq r\}$
and one closed half space
$\{\vz\colon \vn^T\vz\leq c\}$.
Here $\vn$ is the unit normal to the half space and $c\geq 0$.
This yields the dome region $\DM(\vq,r;\vn,c)=\{\vz\colon \vn^T\vz\leq c, \norm{\vz-\vq}{2}\leq r \}$
illustrated in Fig.~\ref{fig:gendome}.

The following features of the dome
$\DM(\vq,r;\vn,c)$ will be useful.
We call the point $\vq_d$ on the bounding hyperplane and the line passing through $\vq$ in the direction of the hyperplane normal  the \emph{dome center}.
The signed distance from $\vq$ to $\vqb$ in the direction $-\vn$ is a fraction $\fr_{d}$ of the radius $r$ of the sphere.
We call the maximum straight line distance $r_d$ one can move from $\vqb$ within the dome and hyperplane
the \emph{dome radius}.
Under the sign convention indicated above,
simple Euclidean geometry gives the following relationships:
\begin{align}
\fr_{d} &= \fracn{(\vn^{T}\vq-c)}{r}, \label{eq:frd}\\
\vqb &= \vq - \fr_{d} r \vn,  \label{eq:vqb}\\
\rd  &= r \sqrt{1-\fr_{d}^{2}}. \label{eq:rb}
\end{align}
To ensure that the dome is nondegenerate (a nonempty and proper subset of each region), we need $\vqb$ to be inside the sphere.
Hence we require  $-1\leq\fr_{d}\leq 1$.
So we need $\vq^{T}\vn \geq c-r$,
this ensures that the intersection is a proper subset of the sphere and the half space; and we need $\vq^{T}\vn \leq c+r$, this ensures the intersection is nonempty.

To find
$\mu(\at)=\max_{\vt\in \DM(\vq,r;\vn,c)} \vt^T\at$, for $\at \in \R^\dd$,
we solve the optimization problem \eqref{eq:optprobm}
with $m=1$.
Particular instances of this problem were solved in
\cite[Appendix A]{Ghaoui2011Safe} (by solving a Lagrange dual problem)
and in \cite[\S3]{Xiang2012Fast} (by directly solving a primal problem).
Both approaches can be extended to solve the general problem \eqref{eq:optprobm} with $m=1$.
This yields the following lemma, and the dome screening test.

\begin{figure}[t]
\centering
\subfigure[]{\includegraphics[width=0.25\textwidth]{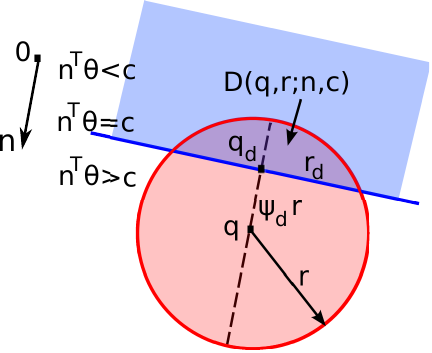}
\label{fig:gendome}}
\subfigure[]{\includegraphics[width=0.3\textwidth]{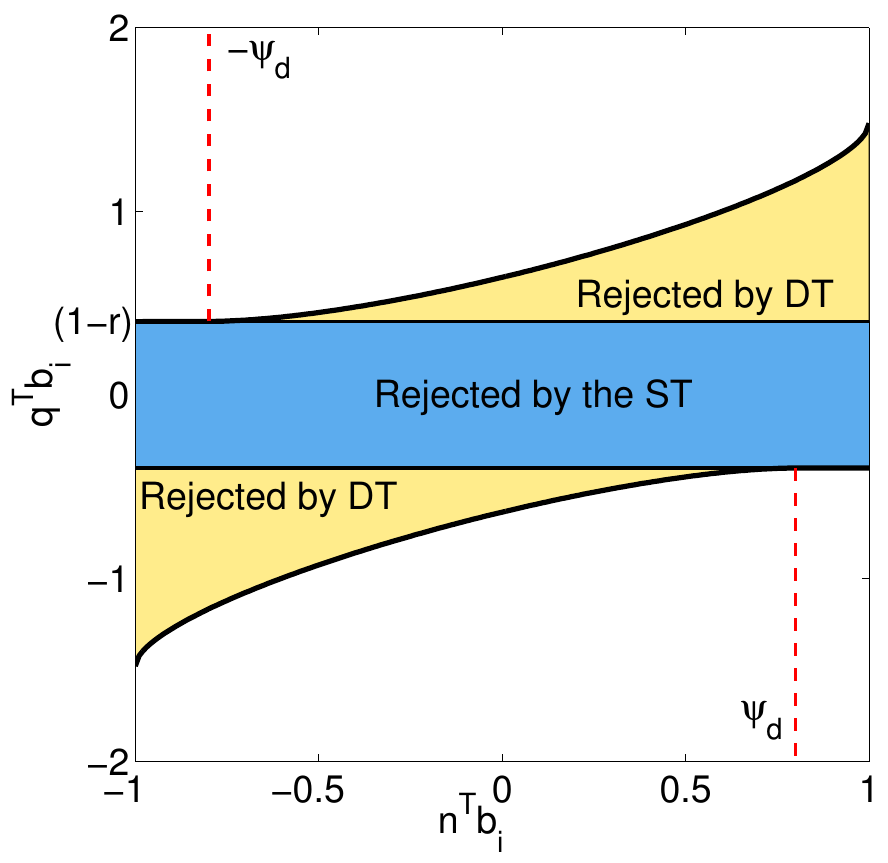}
\label{fig:rejection_area}}
\caption{\small
{\bf (a)} A general dome region $\DM(\vq,r;\vn,c)$ shown for $0<\fr_{d}<1$ and the dome consisting of less than half the sphere.
{\bf (b)} The rejection area (shaded) of a lasso dome test.
}
\vspace{-0.5cm}
\end{figure}

\begin{lemma}\label{lem:max-thTb}
Fix a dome $\DM=\DM(\vq,r ;\vn,c)$ with
$|\fr_{d}| \leq 1$.  Then for $\at\in\R^\dd$,
$$
\mu_{\DM}(\at)
= \vq^{T}\at +\maxu_1(\vn^{T}\at,\|\at\|_2),
$$
where $\maxu_1(t_1,t_2)$ is the function
\begin{align}
\begin{split} \label{eq:defmaxu}
&\maxu_1(t_1,t_2) =\\
&\begin{cases}
r t_2 , & \textrm{if } t_1< -\fr_{d} t_2;\\
-\fr_{d}rt_1 + r\sqrt{t_2^2 - t_1^{2}}\sqrt{1-\fr_{d}^{2}},
	& \textrm{if } t_1 \geq -\fr_{d} t_2.
\end{cases}
\end{split}
\end{align}
\end{lemma}

\begin{theorem}\label{thm:dometest}
The screening test for a nondegenerate dome
$\DM(\vq,r;\vn,c)$ is:
\begin{align}
\begin{split}\label{eq:dt}
	&T_{\DM(\vq,r;\vn,c)}(\at)= \\
	&\!\!\! \begin{cases}
	1, &\!\! \textrm{if }\tl(\vn^T\at,\|\at\|_2) \!<\! \vq^T\at
	\!<\! \tu(\vn^T\at,\|\at\|_2); \\
	0, &\!\!\textrm{otherwise;}
	\end{cases}
\end{split}
\end{align}
where
$	\tu(t_1,t_2) =
	1-\maxu_1(t_1,t_2)$
and for the lasso $\tl(t_1,t_2)=-\tu(-t_1,t_2)$,
and for the nonnegative lasso $\tl(t_1,t_2)=-\infty$.
\end{theorem}

We denote a dome test by $\DT(\vq,r;\vn,c)$.
Although defined piecewise, the functions $\tu$ and $\tl$ in
Theorem \ref{thm:dometest} are continuous and smooth: $\tu,\tl \in C^1$.
This can be checked using simple calculus.
The parameters $r$ and $c$ of the dome do not appear as arguments in the test but play a role through
$\maxu_1$. The test simplifies for unit norm features. In that case,
$t_2=\|\at_i\|_2=1$ and $\maxu_1$, $\tu$ and $\tl$ are only functions of $t_1$.

To gain some insight into this test, consider the situation when $r<1$ and all features have unit norm.
We can factor the test into the composition of two functions:
a linear map $\at_{i}\mapsto [\vq,~\vn]^{T}\at_{i}$
and a two-dimensional decision function $\ddf_{r,\fr_{d}}$ with 
$\ddf_{r,\fr_{d}} (s,t) = 1$ if $\tl(t) < s< \tu(t)$, and $0$ otherwise;
where $s=\vq^T\at_i\in [-\|\vq\|_{2},\|\vq\|_{2}]$,
$t=\vn^T\at_i\in [-1,1]$,
and $\tu(t)$, $\tl(t)$ are given in Theorem \ref{thm:dometest} with
$t=t_1$ and $t_2=1$.
We can display the test rejection region by plotting
$\tl(t)$ and $\tu(t)$ versus $t$ as shown in Fig.~\ref{fig:rejection_area}.
For the lasso, the rejection region has upper and lower boundaries.
The sections of the boundaries with $\tu(t) = (1-r)$
and $\tl(t)=-(1-r)$,  correspond to the sphere test $T_{\SPH(\vq,r)}$.
If feature $\at_{i}$ maps into the shaded region in the figure, then $\at_{i}$ is rejected. The lightly shaded (yellow) area indicates the extra rejection power of the dome test over the underlying sphere test.
For a given value of $\vq^T\at_i>0$, the dome test lowers the bar for rejection as $\vn^T\at_i$ increases.

\subsubsection{Parameter selection}
Now consider the parameter selection problem.
Since we have discussed parameter selection for a spherical bound,
we assume $\SPH(\vq,r)$ is given and give examples of bounding $\vto$ within a suitable half space.

Each constraint of the dual problem $\at^{T}\vt\leq 1$ bounds $\vto$.
This half space has $\vn=\at/\|\at\|_2$ and $c=1/\|\at\|_2$.
The resultant dome is nonempty since both the sphere and the half space contain $\vto$. To ensure it is proper, we require $\vq^{T}\at \geq 1-r\|\at\|_2$. This means that the sphere test does not reject the feature $\at$.
In particular, we can select $\at$ to minimize the disk radius $r_d$.
To do so, we maximize $\fr_{d}$ given by \eqref{eq:frd}:
\begin{equation}\label{eq:bg}
\at_g = \argmax_{\at\in \CP} \frac{\at^T\vq-1}{\|\at\|_2}
\end{equation}
For unit norm features, \eqref{eq:bg} selects
the feature most correlated with $\vq$.
If in addition, $\vq=\tv/\lambda$, then \eqref{eq:bg} yields $\at_g=\atm$.
Selecting the default spherical bound
and using \eqref{eq:bg} gives the specific dome:
\begin{equation}\label{eq:dt(xl;bm)}
\DM(\tv/\lambda, |\fracn{1}{\lambda}-\fracn{1}{\lm}|~\|\tv\|_2;
\at_g/\|\at_g\|_2,1/\|\at_g\|_2).
\end{equation}
We call this the \emph{default dome bound}.
When $\tv$ and all the features have unit norm, this simplifies to
$\DM(\tv/\lambda, |\fracn{1}{\lambda}-\fracn{1}{\lm}|;\atm,1).$
This dome is illustrated in Fig.~\ref{fig:dtest} (left).

If $\vto_{0}$ is the dual solution of an instance $(\tv_{0},\lambda_{0})$, then
$\vto_{0}$ lies on the boundary of $\FS$. Moreover, its {\em optimality} for  $(\tv_{0},\lambda_{0})$ ensures that it satisfies the inequalities \eqref{eq:vtoineq}.
Hence for each $\vt\in \FS$:
\begin{equation} \label{eq:xoloto_hs}
(\tv_{0}/\lambda_{0}- \vto_{0})^{T}\vt  \leq (\tv_{0}/\lambda_{0}-\vto_{0})^{T}\vto_{0}.
\end{equation}
Since $0\in \FS$, the right hand side is nonnegative.
Therefore this inequality bounds $\FS$ in the closed half space
$\vn_{0}^{T} \vt \leq c_{0}$ with
\begin{align}
\begin{split}
r_{0}&=\|\tv_{0}/\lambda_{0} -\vto_{0}\|_{2},\\
\vn_{0}&=(\tv_{0}/\lambda_{0}-\vto_{0})/r_{0},\\
c_{0}&=\vn_{0}^{T}\vto_{0} .
\end{split} \label{eq:chs0}
\end{align}

\begin{figure}[t!]
\begin{center}
\includegraphics[height=4.5cm]{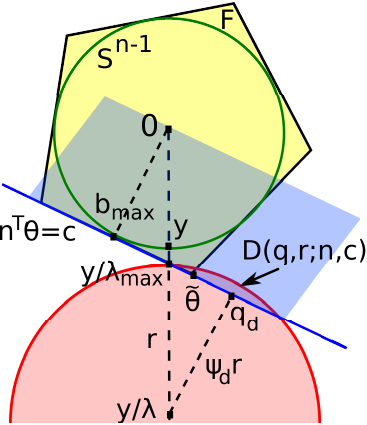}
\includegraphics[height=4.5cm]{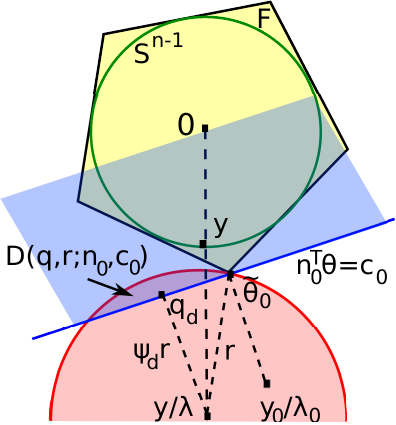}
\end{center}
\caption{\small Two dome tests for unit norm features and target vector.
Left: The dome \eqref{eq:dt(xl;bm)} based on the feasible point $\tv/\lm$
and the closed half space $\atm^{T}\vt \leq 1$.
Right: The dome \eqref{eq:xoloto_dt} based on a solved instance
$(\tv_{0},\lambda_{0},\vto_{0})$.
} \label{fig:dtest}
\vspace{-0.5cm}
\end{figure}

The intersection of this half space with the bounding sphere $\SPH(\vq,r)$
is nonempty and it is proper if $\fr_{d}\geq -1$. To check this condition
note that
\begin{align*}
\fr_{d}
&= \frac{\vn_{0}^{T} \vq -\vn_{0}^{T}\vto_{0}}{r}\\
&= \frac{(\tv_{0}/\lambda_{0}-\vto_{0})^{T}}{r_{0}}\frac{(\vq - \vto_{0})}{\|\vq-\vto_0\|_2}
\frac{\|\vq-\vto_0\|_2}{r}\\
&= \cos\beta \frac{\|\vq-\vto_0\|_2}{r}
\end{align*}
where $\beta$ is the angle between $n_0$ and $\vq-\vto_0$.
So if $\cos\beta >0$ or $\vto_0\in \SPH(\vq,r)$, then the dome is proper.
For example, $\vq=\tv/\lambda$ and
$r=\|\tv/\lambda -\vto_{0}\|_{2}$,
yields the proper dome:
\begin{equation}\label{eq:xoloto_dt}
\DM(\tv/\lambda, \|\tv/\lambda -\vto_{0}\|_{2};\vn_{0},c_{0}).
\end{equation}
This dome illustrated in Fig.~\ref{fig:dtest} (right).

\subsubsection{Connections with the Literature}\label{ssec:priordome}
Specific dome tests
were introduced in \cite[\S2.4]{Ghaoui2011Safe} and \cite[\S3]{Xiang2012Fast}.
The dome test discussed in \cite{Xiang2012Fast} is
based on the default dome bound \eqref{eq:dt(xl;bm)}
for unit norm features and unit norm $\tv$.
The SAFE-LASSO test in \cite[\S2.4]{Ghaoui2011Safe}
is a dome test specifically designed for screening and solving lasso problems at points along the regularization path.
A triple $(\tv,\lambda_{0}, \vto_{0})$ is given where
$\vto_{0}$ is the dual solution for instance $(\tv,\lambda_{0})$.
The test uses this to screen the dictionary for an
instance $(\tv,\lambda)$ with $\lambda<\lambda_{0} \leq \lm$.
We show that the dome employed is \eqref{eq:xoloto_dt} with $\tv_{0}=\tv$.
The solution in \cite[\S2.4]{Ghaoui2011Safe}
entails specifying a bounding sphere and a half space,
then solving the corresponding version of  \eqref{eq:optprobm}.
The selected half space is $g^{T}(\vt-\vto_{0})\geq 0$ where
$g=\nabla G(\vto_{0})=-\tv/\lambda_{0}+\vto_{0}$
is the gradient of the dual objective for the solved
instance evaluated at $\vto_{0}$ (up to positive scaling).
The spherical bound is obtained by scaling $\vto_{0}$ to obtain the closest
feasible solution to $\tv/\lambda$.
This can be specified by letting
$r_{0}=\|\tv/\lambda_{0}-\vto_{0} \|_{2}$
and setting
$\vq =\tv/\lambda$,
$\rh = \min_{s\in [-1,1]} \|s\vto_{0}-\tv/\lambda\|_{2}$,
$\vn =(\tv/\lambda_{0}-\vto_{0})/r_{0}$ and
$c =\vn^{T}\vto_{0}$.
Assume $\lambda \leq \lambda_{0} < \lm$ and
let $\sh(\lambda)$ denote the optimal value of $s$ in the definition of $\rh$.
By the optimality of $\vto_{0}$ for the instance $(\tv,\lambda_{0})$,
we must have $\sh(\lambda_{0})=1$.
In addition, it must hold that $\tv^{T}\vto_{0}\geq 0$
otherwise the feasible point $-\vto_{0}$ would be closer to $\tv/\lambda$.
By simple calculus we then determine that
$\sh(\lambda)=\min\{1, \fracn{(\tv^{T}\vto_{0})}{(\lambda \|\vto_{0}\|_{2}^{2})}\}$.
It follows that for all $\lambda <\lambda_{0}$, $\sh=1$.
Hence for $\lambda<\lambda_{0}$ we can take
$\rh=\|\tv/\lambda -\vto_{0}\|_{2}$.
Thus for $\lambda<\lambda_{0}$, SAFE-LASSO
uses the dome \eqref{eq:xoloto_dt} with the constraint $\tv_{0}=\tv$.
\vspace{-4mm}

\subsection{Iteratively Refined Bounds}\label{sec:refine}
Under favorable circumstances, it is possible to refine a sphere  $\SPH(\vq_,r)$
 bounding $\vto$ to obtain a bounding sphere of smaller radius.
Let the half space $(\vn,c)$ also bound $\vto$ and
its intersection with $\SPH(\vq,r)$
result in a dome  $\DM=\DM(\vq,r;\vn,c)$
with  parameters $\fr_d$, $\vq_d$, and $r_d$.
Since $\DM$ is a bounded convex set, there exists
a unique sphere of smallest radius that bounds $\DM$. This is called the {\em circumsphere} of $\DM$.
We claim that if $0<\fr_d \leq 1$, or equivalently
$\vq \notin \DM$,  then the circumsphere of $\DM$ is
$\SPH(\vq_d,\rd)$.
In this case, $\rd$ is strictly smaller than $r$ and
$\SPH(\vq_d,\rd)$ is a tighter spherical bound on $\vto$. This is summarized below.

\begin{lemma}\label{lem:circumsphere}
Let $\SPH=\SPH(\vq,r)$ and the half space $(\vn,c)$ bound the dual solution $\vto$, with the resulting dome $\DM=\DM(\vq,r;\vn,c)$ satisfying $0<\fr_d \leq 1$.
Then $\SPH(\vq_d, \rd)$ is the circumsphere of $\DM$ and hence bounds $\vto$.
\end{lemma}

If suitable half spaces can be found, e.g., among the vectors in $\CP$, 
the construction in Lemma \ref{lem:circumsphere} can be used iteratively. 
At step $k$, we have a bounding sphere  $\SPH_{k}=\SPH(\vq_{k},r_{k})$
and seek $\at\in \CP$ such that 
$\vn=\at/\|\at\|_2$ and $c=1/\|\at\|_2$ satisfy
\begin{equation}\label{eq:betak+1}
0< \fr_{k} = \fracn{(\vq_k^T\vn -c)}{r_{k}}\leq 1.
\end{equation}
If such $\at$ exists, set $\vn_k=\vn$,
$c_k=c$ and
\begin{align}
\vq_{k+1} &= \vq_{k} -\fr_{k} r_{k} \vn_{k} , \label{eq:qk+1}\\
r_{k+1} &= r_{k} \sqrt{ 1-\fr_{k}^{2}}, \label{eq:rk+1}
\end{align}
to obtain a tighter bounding sphere $\SPH_{k+1}=\SPH(\vq_{k+1},r_{k+1})$.
A greedy strategy selects $\at$ at step $k$ to minimize $r_{k+1}$,
or equivalently to maximize $\fr_{k}$:
\begin{equation}\label{eq:greedyk}
\at^{(k)} = \argmax_{\at\in \CP} \frac{\at^T \vq_k -1}{\|\at\|_2} \ .
\end{equation}
When all features have equal norm, this
reduces to maximizing the inner product of
$\at$ and $\vq_{k}$.
This has a simple interpretation.
$\SPH_{k}$ can be thought of as a location bound on $\vto$ 
with the center $\vq_{k}$ the ``estimate'' of $\vto$ given the bound.
The greedy strategy selects $\at$
by maximizing  its alignment with
the current estimate $\vq_{k}$ of $\vto$.
Since $\vto$ is proportional to the optimal residual in the primal problem (see \eqref{eq:relationship}),
this strategy selects features ``best correlated'' with the current estimate of the optimal residual.

\vspace{-3mm}
\subsection{Are More Half Spaces Worthwhile?}\label{sec:THT}
We have examined region tests defined by the intersection of a bounding sphere and one half space ($m=1$),
and have shown that, in general, these have additional rejection power over the simpler sphere tests ($m=0$).
Are more complex tests worthwhile?
To examine this, we go one step further and examine the
region test defined by the intersection of a bounding sphere and two half spaces ($m=2$).
Examining the relative performance of this test
will allow us to determine where we currently stand in the trade-off between rejection power and computational efficiency. 

Let $\ADM(\vq, r; \vn_1, c_1; \vn_2, c_2)$ denote
the region formed by the intersection of a sphere
$\SPH(\vq, r)=\{\vt:\norm{\vt-\vq}{2}\leq r\}$ and two closed half spaces $H_i=\{\vt:\vn_i^T\vt\leqn c_i\}$, where $\vn_i$ is the the unit normal to $H_i$ and $c_i\geqn 0$, $i=1,2$.
We call the corresponding screening test a \emph{Two Hyperplane Test} (THT).

Each half space $H_i=\{\vt\colon \vn_i^T\vt \leq c_i\}$
intersects the sphere forming a dome with parameters
$\fr_i=(\vn_i^T\vq-c_i)/r$,
$\vq_i=\vq-\fr_ir\vn_i$, and
$r_i=r\sqrt{1-\fr_i^2}$, $i=1,2$.
To ensure each intersection $H_i\cap\SPH(\vq,r)$ is nonempty and proper,
we need $-1\leq \psi_i \leq1$, $i=1,2$,
and to ensure the two half spaces intersect within the sphere,
we need $\arccos\psi_1+\arccos \psi_2 \geq \arccos(\vn_1^T\vn_2)$.
Under these conditions, $\ADM(\vq, r; \vn_1, c_1; \vn_2, c_2)$
is a nonempty, proper subset of the sphere and each half space.

To find $\mu_{\ADM}(\at)=\max_{\at\in \ADM} \vt^T\at$
we solve the optimization problem \eqref{eq:optprobm} with $m=2$.
Using standard techniques, this problem can be solved in closed form
yielding the expressions for $\mu_{\ADM}$ in the following lemma.
The corresponding test then follows from \eqref{eq:rrt}.

\begin{lemma}\label{lem:dmax}
Fix the region $\ADM=\ADM(\vq, r; \vn_1, c_1; \vn_2, c_2)$ and
let $\fr_i$ satisfy $|\fr_i|\leq1$, $i=1,2$, and $\arccos\psi_1+\arccos \psi_2 \geq \arccos(\vn_1^T\vn_2)$.
Let $h(x,y,z)=\sqrt{(1-\nc^2)z^2 +2\nc xy-x^2-y^2}$, where $\nc=\vn_1^T\vn_2$.
Then for $\at\in \R^{\dd}$,
\begin{equation}\label{eq:dmax}
\mu_{\ADM}(\at)=
\vq^T\at+\maxu_2(\vn_1^T\at,\vn_2^T\at,\|\at\|_2)
\end{equation}
where
\begin{align*}
&\maxu_2(t_1, t_2, t_3) =\\
&\left\{
\begin{array}{ll}
    rt_3, & \text{if } (a);\\
    -rt_2\fr_2+r\sqrt{t_3^2-t_2^2}\sqrt{1-\fr_2^2}, & \text{if } (b);\\
    -rt_1\fr_1+r\sqrt{t_3^2-t_1^2}\sqrt{1-\fr_1^2}, & \text{if } (c);\\
    -\frac{r}{1-\nc^2}\left[(\fr_1-\nc\fr_2)t_1+(\fr_2-\nc\fr_1)t_2\right]+\\
    \quad \frac{r}{1-\nc^2}h(\fr_1,\fr_2,1)h(t_1,t_2,t_3),
    &    \text{otherwise;}
    \end{array}
    \right.
\end{align*}
and conditions $(a),(b),(c)$ are given by
\begin{equation*}
\begin{array}{ll}
\hskip-0.2cm (a) &t_1 \lessn -\fr_1 t_3 \ \&\  t_2 \lessn -\fr_2 t_3;\\
\hskip-0.2cm (b) &t_2\geqn -\fr_2 t_3\  \&\ \\
&  \fracn{(t_1-\nc t_2)}{\sqrt{t_3^2-t_2^2}} \lessn \fracn{(-\fr_1+\nc\fr_2)}{\sqrt{1-\fr_2^2}};\\
\hskip-0.2cm (c) &t_1\geqn -\fr_1 t_3 \ \& \ \\
&\fracn{(t_2-\nc t_1)}{\sqrt{t_3^2-t_1^2}} \lessn \fracn{(-\fr_2+\nc\fr_1)}{\sqrt{1-\fr_1^2}}.
\end{array}
\end{equation*}
\end{lemma}

\begin{theorem}\label{thm:THT}
The Two Hyperplane Test (THT) for the region
$\ADM(\vq, r;\vn_1, c_1;\vn_2, c_2)$ is:
\begin{align}
\begin{split}\label{eq:THT}
	T_{\ADM}(\at_{i})=
	\begin{cases}
	1, & \textrm{if } (a')  \\
	0, & \textrm{otherwise;}
	\end{cases}
	\end{split}
\end{align}
where condition (a') is
$$
V_l(\vn_1^T\at_i, \vn_2^T\at_i,\|\at_i\|_2)\lessn
	\vq^T\at_i\lessn V_u(\vn_1^T\at_i, \vn_2^T\at_i,\|\at_i\|_2);
$$
with $V_u(t_1, t_2, t_3) = 1-\maxu_2(t_1,t_2,t_3)$
and for the lasso, $V_l(t_1, t_2,t_3)=-V_u(-t_1, -t_2, t_3)$, and for the nonnegative lasso, $V_l(t_1, t_2,t_3)=-\infty$.
\end{theorem}

Theorem \ref{thm:THT} indicates that THT
uses only the $3p$ correlations
$\{\vq^T\at_i, \vn_1^T\at_i, \vn_2^T\at_i\}_{i=1}^p$.
So the test has time complexity $O(pn)$.

\subsubsection{Parameter selection}
Assume the sphere $\SPH(\vq,r)$ has been selected.
The inequality constraints in \eqref{eq:dual}
provide the natural half space bounds $\vto\in H(\at)$, $\at\in \CP$.
$H(\at)$ can be equivalently specified as
$\{\vt\colon \vn^T \vt\leq c\}$ with $\vn=\at/\|\at\|_2$
and $c=1/\|\at\|_2$ and the resultant dome
$H(\at)\cap \SPH(\vq,r)$ has parameters given by \eqref{eq:frd}, \eqref{eq:vqb} and \eqref{eq:rb}.

We seek two such half spaces. We can select the first
by minimizing its dome radius $r_{d}$. By \eqref{eq:rb},
this requires maximizing $\fr_{d}$:
\begin{equation}\label{eq:firstb}
\at^{(1)}= \argmax_{\at\in \CP}\ \
\frac{\at^T \vq -1}{\|\at\|_2}.
\end{equation}
When all features have equal norm,
we can simply maximize $\at^T\vq$ over
$\at\in \CP$.

Suppose we have selected the first feature $\at^{(1)}$ using \eqref{eq:firstb}.
This yields a dome with dome center $\vq^{(1)}=\vq_{d}$
and dome radius $r^{(1)}=r_{d}$.
Assume that $\fr_{d} \geq 0$.
Then by Lemma \ref{lem:circumsphere},
the smallest sphere containing the dome has center $\vq^{(1)}$ and radius $r^{(1)}$.
To select the second feature, we can focus on the sphere $S(\vq^{(1)}, r^{(1)})$ and repeat the above construction:
\begin{equation}\label{eq:scndb}
\at^{(2)}= \argmax_{\at\in \CP/\at^{(1)} }
\frac{\at^T \vq^{(1)} -1}{\|\at\|_2}.
\end{equation}
When all features have equal norm, we can simply
maximize  $\at^T\vq^{(1)}$ over $\at\in \CP/\at^{(1)}$.
We call this parameter selection method
\emph{Dictionary-based THT} (D-THT).

Alternatively, if we have solved the instance $(\tv_0, \lambda_0)$ yielding primal and dual solutions $\wvo_0$ and $\vto_0$ (see \eqref{eq:relationship}),
then $\vto_0$ must satisfy the inequalities
\eqref{eq:vtoineq}.
Using some algebra and \eqref{eq:relationship},
these inequalities can be written as:
$(B\wvo_0)^T\vt \leq (B\wvo_0)^T \vto_0$.
Since $0\in \FS$, the right hand side is nonnegative.
Hence
the inequality bounds $\FS$ in the half space $\vn_1^T\vt \leq c_1$ with
\begin{equation}\label{eq:hypc}
\vn_1=\Dict\wvo_0/\norm{\Dict\wvo_0}{2}, \qquad c_1=\vn_1^T\vto_0.
\end{equation}
One can then select $\vn_2$ and $c_2$ using \eqref{eq:scndb}.

We will return to the THT tests in \S\ref{sec:exp} where we compare the performance of the tests with $m=0,1,2$ and examine the trade-off between rejection rate and computational efficiency that increasing $m$ imposes.

\subsubsection{Connections with the Literature}\label{ssec:THT}
The form of the Two Hyperplane Test was first presented (without proof) in \cite{YunWang2013a}, for unit norm features and target vector.
The form given here (with proofs) is a generalization of that result.
The general formulation allows the use of any sphere and hyperplane constraints bounding $\vto$ and includes the feature constraint used in \cite{Xiang2012Fast}
as a special case.

\subsection{Composite Tests} \label{sec:comp}

The construction described in \S\ref{sec:refine}
gives rise to a finite sequence of spheres and domes:
$
\SPH_{1} \supset \DM_{1} \subset \SPH_{2}
\supset \cdots \supset
\SPH_{k-1}\supset \DM_{k-1} \subset \SPH_{k}
$.
Each sphere and dome has an associated test.
But since $\DM_{j}$ is contained in $\SPH_{j}$ and $\SPH_{j+1}$, each dome test is stronger
than the tests for the spheres that precede and succeed it.

But $\SPH_{j+1}$ is not contained in $\SPH_{j}$ and $\DM_{j+1}$ is not contained in $\DM_{j}$. So we can't claim that the last dome $\DM_{k-1}$ leads to the strongest test. Moreover, a test based on the region
$\cap_{j=1}^{k-1} \DM_{j}$ is usually too complex to compute.

An alternative is to implement a composite test that rejects $\at_{i}$
if it is rejected by any of the tests $\{T_{\DM_{j}}\}_{j=1}^{k-1}$.
For the nonnegative lasso, $T_{\DM_{j}}$ takes the form $\mu_{j}(\at_{i}) <1$, with
$\mu_{j}(\at_{i})
=\vq_j^T\at_i + \maxu_1(\vn_{j}^{T}\at_{i},\|\at_i\|_2)$
and $\maxu_1$ given by \eqref{eq:defmaxu}.
So the composite test rejects $\at_{i}$ if
\begin{equation}\label{eq:irdt_nn}
\min_{j=1:k}\{\  \vq_j\at_i
+\maxu_1(\vn_{j}^{T}\at_{i},\|\at_{i}\|_2)\ \} <1.
\end{equation}
Similarly, for the lasso problem the composite test rejects $\at_{i}$ if
\begin{align}
\begin{split}\label{eq:irdt}
\min_{j=1:k} \{\
\max\{
&\vq_j^T\at_i+\maxu_1(\vn_{j}^{T}\at_{i},\|\at_{i}\|_2),\\
&-\vq_j^T\at_i-\maxu_1(-\vn_{j}^{T}\at_{i},\|\at_{i}\|_2 \} \ \} <1.
\end{split}
\end{align}
Reflecting the dome construction method, we call the tests \eqref{eq:irdt_nn} and \eqref{eq:irdt} \emph{iteratively refined dome tests} (IRDT).
These tests can be implemented in several ways and extra domes arising in the course of the construction can also be included. This is illustrated in \S\ref{sec:alg}.
The major cost of the tests is
calculating the inner products $\vq_{j}^{T}\at_i$ and $\vn_j^T\at_i$
for each feature $\at_i$ to be tested.
Because of the iterative construction, this can be done by computing
$\vq_{1}^{T}\at_i, \vn_{1}^{T}\at_i, \dots, \vn_{k-1}^{T}\at_i$
(see \eqref{eq:betak+1}, \eqref{eq:qk+1}, \eqref{eq:greedyk}).
So to execute all of the tests $\DM_{1},\dots,\DM_{k}$,
only $k$ inner products are used per feature tested.
This is $O(\dd k)$ time complexity per feature tested
where $\dd$ is the feature dimension.
So the marginal cost of increasing $k$ by $1$
is the cost of computing one additional inner product per feature tested.

A composite test is mathematically equivalent to test disjunction,
$(T_{1}\disjct T_{2})(\at_{i}) = T_{1}(\at_{i})\disjct T_{2}(\at_{i})$.
A disjunction of region tests is weaker than the test based on the intersection of the regions. For example, consider two spheres of equal radius with a small intersection. Both spheres can intersect a half space while the intersection does not.

\begin{lemma}\label{lem:conjt}
For compact sets $\RR_{1}$, $\RR_{2}$:
	$T_{\RR_1} \disjct T_{\RR_2} \comp T_{\RR_1 \cap \RR_2}.$
\end{lemma}
Lemma \ref{lem:conjt} indicates that a disjunction of tests is trading rejection performance for simplicity and ease of implementation.
Despite the above limitation, the IRDT test is very  competitive
with Dictionary-based THT on the datasets used in our numerical studies.

\subsubsection{Connections with the Literature}
The sphere test ST3 in \cite{Xiang2011Learning_b} is
based on a refined spherical bound.
In \cite{Xiang2011Learning_b} it is
assumed that $\tv$ and all features have unit norm.
ST3 is then constructed starting with the default spherical bound
$\SPH(\vq_{1},r_{1})$ with $\vq_{1}=\tv/\lambda$
and $r_{1}=(\fracn{1}{\lambda}-\fracn{1}{\lm})$.
The greedy strategy selects the feature $\at=\atm$.
Then the dual solution $\vto$ lies in the default dome formed by the intersection of the spherical ball $\SPH(\vq_{1},r_1)$ and the half space $\ch(\atm)$.
This intersection is indicated by the green dome region $\GR$ in Fig.~\ref{fig:3ST}.
The smallest spherical ball bounding $\GR$ (dashed magenta circle in Fig.~\ref{fig:3ST}) is obtained by substituting the values of
$\vq_{1}$, $r_{1}$ and $\atm$ into
\eqref{eq:betak+1}, \eqref{eq:qk+1} and \eqref{eq:rk+1}.
This yields $\fr_{2} = \lm$,
$\vq_{2} = \fracn{\tv}{\lambda} - \left(\fracn{\lm}{\lambda}- 1\right)\atm$
and
$r_{2} = \sqrt{\fracn{1}{\lambda^{2}_{\max}}-1}\left(\fracn{\lm}{\lambda}- 1\right)$.
These  parameters are derived in \cite{Xiang2011Learning_b} using a distinct approach.

A two term disjunction test is used in \cite{ellpscr2012}. This test is implemented sequentially. The first test is applied and then the second is applied to the remaining features. Any disjunction test can be implemented sequentially in this fashion. The key innovation in
\cite{ellpscr2012} is that each test is based on an
ellipsoidal bound on $\vto$.  The first ellipsoid is the minimum volume ellipsoid containing the default dome \eqref{eq:dt(xl;bm)}.
The second ellipsoid is constructed in a greedy fashion by selecting a feature so that the best ellipsoidal bound of the intersection of its half space and the first ellipsoid has minimum volume. The first step is in the spirit of ST3 except using an ellipsoidal bound.
The second step is bound refinement based on ellipsoids rather than spheres.
An ellipsoidal bound is tighter than the spherical bounds used in this section.
However, its description requires a center
$\vq\in \R^{\dd}$ and a matrix $P\in \R^{\dd\by \dd}$ to encode its shape and orientation. When $\dd$ is large this could be an impediment.
In contrast, a sphere requires a center $\vq\in \R^{\dd}$ and a scalar radius $r$.

\section{Sequential Screening}\label{sec:sarp}
The screening tests discussed so far screen the dictionary once, then solve the reduced lasso problem. We hence call these tests ``one-shot'' screening tests. These tests can perform well for moderate to large values of $\lambda/\lm$ but often fail to provide adequate rejection performance for smaller values of $\lambda/\lm$. This is primarily due to the challenge of obtaining a tight region bound on $\vto$ when $\lambda/\lm$ is small.

Alternative screening methods can help with this problem.
For example, \cite{Ghaoui2010Safe} examined the idea of screening and solving \eqref{eq:iterW} for a sequence of instances $\{(\tv, \lambda_k)\}_{k=1}^N$ (Recursive-SAFE).
At step $k$ the previously solved instance
$(\tv, \lambda_{k-1})$ defines a bound on the dual solution of the instance $(\tv, \lambda_k)$.
Hence the previous solution can help screen the next instance in the sequence. A similar idea is proposed by \cite{Tibshirani2010Strong} in the form
of the Strong Sequential Rule. This is used to solve the lasso problem ``over a grid of $\lambda$ values''.
In \cite{Ghaoui2011Safe}, the SAFE test for the lasso is upgraded to use a specific dome test.

In a similar spirit, \cite{dpp2015} proposed running a homotopy algorithm to find a solution at the $K$-th breakpoint on the regularization path of $\wvo(\lambda)$.
This effectively solves a sequence of lasso problems
(via homotopy) to obtain a solution $\wvo_K$ at $\lambda_K>\lat$. The dual solution $\vto_K$ is then used to screen the instance $(\tv,\lat)$.
This has potential advantages, but relinquishes control of the values $\lambda_k$ to the breakpoints in the homotopy algorithm. In the worst case the regularization path can have  $O(3^\ncw)$ breakpoints \cite{MairalYu2012}. As a variant on homotopy, Sequential Lasso \cite{ShanLuo2012} solves a sequence of partially $\ell_1$ penalized least squares problems where features with non-zero weights in earlier steps are not penalized in subsequent steps.

With the exception of homotopy, all of the above sequential schemes use a fixed open loop design for $N$ and the sequence $\{\lambda_k\}_{k=1}^N$.
For example, first fix $N\geq 2$, then select $\lambda_1<\lm$, $\lambda_N=\lat$, and
let the intermediate values be selected via geometric spacing: $\lambda_k=\alpha \lambda_{k-1}$ with $\alpha = (\lat/\lambda_1)^{1/(N-1)}$.
To solve instance $(\tv,\lat)$, we first screen and solve the instance $(\tv,\lambda_1)$.
Then sequentially for $k=2,\dots,N$, we screen instance $(\tv,\lambda_k)$ with the help of the known solution of the previous instance, and then solve the reduced problem.
This continues until the solution for $(\tv,\lambda_t)$ is obtained.
Sometimes all solutions on the grid of $\lambda$ values are of interest, e.g., cross validation for parameter selection.
But there are many other applications where only the solution of the final instance $(\tv,\lambda_t)$ is of interest -- the other instances are merely waypoints in the computation.

The solution of the previous instance
helps screen the next instance as follows.
First  use $\vto_{k-1}$ as a dual feasible point to
form the basic bounding sphere \eqref{eq:sb0} for $\vto_k$  
with center $\tv/\lambda_k$ and radius
$\|\tv/\lambda_k - \vto_{k-1}\|_2$.
Then use $\vto_{k-1}$ as the projection of $\tv/\lambda_{k-1}$
onto $\FS$, to form the bounding halfspace \eqref{eq:xoloto_hs}
with
$\vn_{k-1}=(\tv/\lambda_{k-1}-\vto_{k-1})/\|\tv/\lambda_{k-1}-\vto_{k-1}\|_{2}$, and $c_{k-1} = \vn_{k-1}^{T} \vto_{k-1}$.
This sphere and halfspace yield the
bounding dome derived in \S\ref{ssec:priordome}:
\begin{equation}\label{eq:domek_nc}
\DM_{k}=\DM(\tv/\lambda_{k}, r_k; \vn_{k-1}, c_{k-1})
\end{equation}
This dome, illustrated in Fig.~\ref{fig:smla},
encapsulates information about $\vto_k$ provided by the dual solution of the previous instance.

\begin{figure}[th]
\begin{center}
\includegraphics[width=0.75\linewidth]{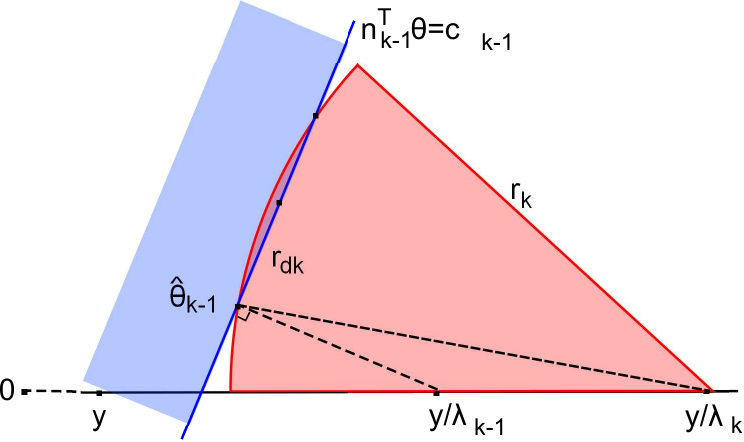}
\end{center}
\caption{\small An illustration of the dome \eqref{eq:domek_nc} formed
at step $k$.
}
\label{fig:smla}
\end{figure}

In contrast to an open loop design, one can use feedback to adaptively select $N$ and the sequence $\{\lambda_k\}_{k=1}^N$ as the computation proceeds \cite{FCSS2013}.
This allows the value of $N$ and the sequence $\{\lambda_k\}_{k=1}^N$ to be adapted to each particular instance. For some instances, a small value of $N$ is used, for others, a larger value is used.  One way to see why feedback helps is to examine the diameter of $\DM_k$ in \eqref{eq:domek_nc}.

\begin{proposition}[From  \cite{FCSS2013}] \label{pro:diamDk}
Let $\DM_k$ be the dome \eqref{eq:domek_nc} and $\delta_k=\diam(\DM_k)$.
Then
\begin{equation}\label{eq:diam}
\delta_k = 2\left (
\frac{1}{\lambda_k}-\frac{1}{\lambda_{k-1}}
\right)
\sqrt{\tv^T(I-\vn_{k-1}\vn_{k-1}^T)\tv} .
\end{equation}
\end{proposition}

Using Proposition \ref{pro:diamDk}, it can be shown that
the data adaptive feedback selection rule
\begin{align}\label{eq:updater}
\frac{1}{\lambda_k}
=  \frac{1}{\lambda_{k-1} }+
\frac{\half R}{\sqrt{\tv^T(I-\vn_{k-1}\vn_{k-1}^T)\tv} } \ ,
\end{align}
where $R>0$ is a selectable parameter,
ensures that $\diam(\DM_k)\leq R$ for all $k>1$.
This allows direct control of how tightly the dome
\eqref{eq:domek_nc} bounds $\vto_k$, $k=1,\dots,N$.
This is called Data-Adaptive Sequential Screening (DASS). 
Note that in this scheme  $N$ is not predetermined. 
Instead the stopping time is decided by the feedback scheme. 
However, \eqref{eq:updater} ensures that
\begin{equation}
\label{eq:dassN}
N\leq 1+ \fracn{\log(1/\lat)}{\log(1+R/2C)},
\end{equation}
where $C$ is an bound on the dual regularization path
\cite{FCSS2013}.
We employ DASS to demonstrate the  effectiveness of sequential screening in \S\ref{sec:exp}. In particular,  Fig.~\ref{fig:fixedR} shows the range of $N$ used by DASS on two datasets and Fig.~\ref{fig:application_impact}  shows its performance in three sparse classification problems.

\section{Algorithms}\label{sec:alg}

\begin{figure}[t]
\begin{algorithm}[H]
\begin{algorithmic}[1]
\small
	\REQUIRE Required: $\{\at_1,\at_2,\ldots,\at_{\ncw}\}$, $\tv$, $\lambda$.\\
	Optional:  $\{\beta_1,\dots,\beta_{\ncw}\}$ with $\beta_i=\|\at_i\|_2$;\\
Optional: 	$\vtf \in \FS$; \\
Optional:    $(\lambda_2, \vtf)$ a dual solution.
	\ENSURE $v_1,v_2,\ldots,v_{\ncw}$ (if $v_i=1$, $\at_i$ is rejected).
	\IF{$\{\beta_i, i=1,\dots,\ncw\}$ is not provided}
		\STATE $\beta_i \gets {\|\at_i\|_2}, i=1,\dots, \ncw$
	\ENDIF
	\STATE $\vq \gets \fracn{\tv}{\lambda}.$	(sphere center)
	\STATE $\rho_{i} \gets {\vq^T\at_i}, 1\leqn i \leqn {\ncw}.$
	\IF{a dual solution is provided}
		\STATE $\vn_1\gets {(\tv/\lambda_2 - \vtf)/\|\tv/\lambda_2 - \vtf\|_2}$.
		\STATE $c_1 \gets {\vn_1^T\vtf}$.	
	\ELSE
		\IF{$\vtf$ is not provided}
			\STATE $\lm \gets \lambda \max_i \{f(\rho_i)\}.$
			\STATE $\vtf \gets \fracn{\tv}{\lm}$.		
		\ENDIF
		\STATE $i_{\ast}\gets \argmax_i \{\fracn{(f(\rho_i)-1)}{\beta_i}\}.$
		\STATE $\vn_1\gets {\at_{i_{\ast}}/\beta_{i_{\ast}}}$.
		\STATE $c_1 \gets {1/\beta_{i_{\ast}}}$.				
	\ENDIF		
	\STATE $r \gets \norm{\vtf-\fracn{\tv}{\lambda}}{2}$. (sphere radius)		
	\STATE $a \gets {\vn_1^T\vq-c_1}$. 	
	\STATE $\sigma_{i} \gets {\vn_1^T\at_i}, 1\leqn i \leqn {\ncw}.$
	\STATE $t_i \gets{\rho_i- a \sigma_i}, 1\leqn i \leqn {\ncw}.$	
	\IF{a dual solution is provided}
		\STATE $j_{\ast}\gets \argmax_i \{\fracn{(f(t_i)-1)}{\beta_i}\}.$	
	\ELSE
		\STATE $j_{\ast}\gets \argmax_{i\neq i_{\ast}} \{\fracn{(f(t_i)-1)}{\beta_i}\}.$	
	\ENDIF
	\STATE $\vn_2 \gets {\at_{j_{\ast}}/\beta_{j_{\ast}} }$.
	\STATE $c_2 \gets {1/\beta_{j_{\ast}}}$.
	\STATE $\tau_i \gets{\vn_2^T\at_i}, 1\leqn i \leqn {\ncw}.$	
	\STATE $v_i \gets \ind{ \tl(\sigma_i,\tau_i,\beta_i)
	<\rho_{i}<  \tu(\sigma_i,\tau_i,\beta_i)}$ 	
\\
{$\empty$ }
\end{algorithmic}
\caption{Two Hyperplane Test (THT)} 
\label{alg:THT}
\end{algorithm}
\caption{
{\small Algorithm for THT. The functions $\tu$ and $\tl$ are from Theorem \ref{thm:THT}.
Other Notation:
For the lasso, $f(z)=|z|$ and $g(z)=\sgn(z)$ and
for the nonnegative lasso, $f(z)=g(z)=z$.
For a logical condition $c(\cdot)$, $\ind{c(z)}$ evaluates to $\true$
if $z$ satisfies condition $c$ and $\false$ otherwise.
}
}
\end{figure}	
\begin{figure}[t]
\begin{algorithm}[H]
\begin{algorithmic}[1]
\small
	\REQUIRE Required: $\{\at_1,\at_2,\ldots,\at_{\ncw}\}$, $\tv$, $\lambda$, $s$\\
	 For simplicity, assume $\norm{\tv}{2}=\norm{\at_i}{2}=1$. \\
	 Optional:  $\vtf\in \FS$.
	\ENSURE $v_1,v_2,\ldots,v_{\ncw}$ (if $v_i=1$, $\at_i$ is rejected).
	\STATE $\vq_1 \gets \fracn{\tv}{\lambda}.$
	\STATE $\rho_{i,1} \gets {\vq_1^T\at_i}, 1\leqn i \leqn {\ncw}.$
	\IF{$\vtf$ is not provided}
		\STATE $\vtf \gets \fracn{\tv}{(\lambda\max_i f(\rho_{i,1}) )}$.
	\ENDIF	
	\STATE $r_1 \gets \norm{\vtf-\fracn{\tv}{\lambda}}{2}$.
	\STATE $v_i \gets \ind{f(\rho_{i,1}) < 1-r_1}, 1\leqn i\leqn \ncw$.
	\STATE $\sigma_i\gets \false, 1\leqn i\leqn \ncw$
	\FOR{$j_1=1,2,\ldots,s$}
		\STATE $h \gets  \argmax_{v_i = \false, \sigma_i = \false} f(\rho_{i, j_1})$.
		\STATE $\at \gets g(\rho_{h,j_1})\at_h$.
		\STATE $t_i \gets \at^T\at_i, 1\leqn i\leqn \ncw$.
		\STATE $\fr \gets \fracn{(f(\rho_{h,j_1})-1)}{r_{j_1}}$.
		\IF{$\fr\leq 0$}
			\STATE BREAK.
		\ENDIF		
		\IF{$j_1<s$}
			\STATE $\vq_{j_1+1} \gets \vq_{j_1} - \fr r_{j_1} \at$.
			\STATE $\rho_{i,j_1+1} \gets \rho_{i,j_1} - \fr r_{j_1} t_i,
			1\leqn i \leqn \ncw.$			
			\STATE $r_{j_1+1} \gets r_{j_1} \sqrt{1-\fr^2}$.
		\ENDIF
		\FOR{$j_2 = j_1, j_1\!-\!1, \ldots, 1$}
			\IF{$j_2<j_1$}
				\STATE $\fr \gets \fracn{(\vq_{j_2}^T\at-1)}{r_{j_2}}$.
			\ENDIF
			\STATE $r\gets r_{j_2}$.
			\FOR{$i\in\{i:v_i=\false\}$}
				\STATE $v_i \gets \ind{ \tl(t_i) <\rho_{i,j_2}<  \tu(t_i)}$ 		
			\ENDFOR
		\ENDFOR
		\STATE $\sigma_h\gets\true$. 	
	\ENDFOR
\end{algorithmic}
\caption{Iteratively Refined DT} 
\label{alg:IRDT}
\end{algorithm}
\caption{
{\small Algorithm for IRDT. Here $\tu$ and $\tl$ are from Theorem \ref{thm:dometest}. For other notation see the caption of Algorithm \ref{alg:THT}.
}
}
\end{figure}

\begin{figure}[t]
\begin{algorithm}[H]
\caption{Data-Adaptive Sequential Screening} 
\begin{algorithmic}[1]
\small
\REQUIRE $\{\at_1,\at_2,\ldots,\at_{\ncw}\}$,
$\tv$,
$\lat$,
$R>0$,
a lasso solver $\mathbb{S}$.
For simplicity, assume $\norm{\tv}{2}=\norm{\at_i}{2}=1$.
\ENSURE $\wvo_{t}$
\STATE $\lm \gets \max_i |\tv^T\at_i|.$
\STATE $k \gets 1$, $\lambda_1 \gets 0.95\lm$.
\STATE call THT with
	$\{\at_1,\at_2,\ldots,\at_{\ncw}\}$, $\tv$, $\lambda_1$ only.
\STATE call $\mathbb{S}$ to solve the lasso problem ($\tv, \lambda_1$) using the
	non-rejected features, get $\wvo_1$.
\STATE $\vto_1 \gets \fracn{(\tv-[\at_1,\at_2,\ldots,\at_{\ncw}]\wvo_1)}{\lambda_1}$.
\WHILE{$\lambda_k >\lat$}
	\STATE $k \gets k+1$.
	\STATE $\vn_{k-1} \gets \frac{\tv/\lambda_{k-1}-\vto_{k-1}}{\norm{\tv/\lambda_{k-1}-\vto_{k-1}}{2}}$.
	\STATE $\frac{1}{\lambda_k} \gets  \frac{1}{\lambda_{k-1} }+\frac{\half R}{\sqrt{\tv^T(I-\vn_{k-1}\vn_{k-1}^T)\tv} }$.
	\IF{$\lambda_k < \lat$}
		\STATE $\lambda_k \gets \lat$.
	\ENDIF
	\STATE call THT with $\{\at_1,\at_2,\ldots,\at_{\ncw}\}$, $\tv$, $\lambda_k$, and a dual solution ($	\lambda_{k-1}$, $\vto_{k-1}$).	
	\STATE call $\mathbb{S}$ to solve the lasso problem ($\tv, \lambda_k$) using the non-rejected features, get $		\wvo_k$.	
	\STATE $\vto_k \gets \fracn{(\tv-[\at_1,\at_2,\ldots,\at_{\ncw}]\wvo_j)}{\lambda_k}$.
\ENDWHILE
\STATE $\wvo_t \gets \wvo_k$.
\end{algorithmic}
\end{algorithm}
\vspace{-2.5mm}
\caption{\small Algorithm for Data-Adaptive Sequential Screening.} \label{alg:smla}
\vspace{-5mm}
\end{figure}

Each of the screening tests previously described
requires the inputs  $\Dict$, $\tv$, and $\lambda$
and returns $v_1,\ldots,v_\ncw$ where $v_i$ is a logical value indicating if $\at_i$ is rejected.  The algorithms can be implemented in an online fashion with very few features stored in memory at once. The critical computation is calculating inner products of the form $\tv^T\at_i$ and $\vn^T\at_i$. 
It follows that the time complexity of one-shot screening is $O(np)$.
If the features are sparse  then running times are further reduced. 
Let $s$ denote the average feature sparsity. Then the time complexity of one shot 
screening is $O(sp)$. 
Reference \cite{FCSS2013} discusses the complexity of Data Adapted Sequential Screening and provides the upper bound \eqref{eq:dassN} on the number of steps used.

A basic implementation of THT is shown in Algorithm \ref{alg:THT}. If the dictionary is unnormalized, the feature norms can be precomputed and passed as an input to the algorithm. If it is normalized, we recommend simplifying the algorithm by setting $\beta_i=1$ and removing unnecessary floating point operations (see \S\ref{sec:THT}).
The algorithm accepts two additional optional inputs:
either a dual solution $(\vtf,\lambda_2)$
or a feasible point $\vtf$.
The dual solution is useful for the application of THT in sequential screening.
It is used to select the first half space used by THT (lines 7,8).
If only a feasible point is provided, it is used to select the sphere radius. Otherwise, the default point $\tv/\lm$ is used.
The remaining half spaces are selected using dictionary-based selection \eqref{eq:firstb}, \eqref{eq:scndb} (\S\ref{sec:THT}).
The output values $v_i$ are determined for each $\at_i$ by evaluation of the THT test in Theorem \ref{thm:THT}.

A basic implementation of IRDT is shown in Algorithm \ref{alg:IRDT}. To keep the notation simple and the algorithm understandable, all features and $\tv$ are assumed to have unit norm, but this is not required (see \S\ref{sec:comp}).
IRDT uses at most $s$ iterations with the value of $s$ supplied by the user (we recommend $s\leq 5$).
The algorithm passes through the dictionary at most $s+1$ times
with the main loop executed at most $s$ times.
The break at line 15 terminates this loop early
if suitable domes can't be found.
The algorithm accepts a feasible point $\vtf$ for the dual problem as an optional input
and can be adapted to accept a known dual solution.
The value $v_i$ for each $\at_i$ is determined by a disjunction  of a set of dome tests each based on the dome test in Theorem \ref{thm:dometest}. These disjunctions are computed sequentially with subsequent tests applied only to currently surviving features.

Data-Adaptive Sequential Screening solves $N$ lasso instances $\{(\tv,\lambda_k)\}_{k=1}^N$ for a sequence of descending values $\lambda_k$ where $\lm> \lambda_1$ and $\lambda_{N}=\lat$ is the regularization parameter value for the instance to be solved.
The user must specify a radius $R>0$.
At each step, the algorithm uses a strong ``one-shot'' screening test, for example THT, provided with a solution of the previous instance, followed by an external lasso solver to solve the screened current instance.
The algorithm sets $\lambda_1=0.95 \lm$ and thereafter uses the feedback rule \eqref{eq:updater} to select $\lambda_k$ until $\lambda_k \leq \lat$. It then sets $N=k$, $\lambda_N=\lat$ and screens and solves the final problem. See \cite{FCSS2013} for additional details on this algorithm.

\begin{table}[t!]
  \centering
  \begin{tabular}{|l|c|c|c|}
  \hline
Data Set & $\ncw$ & $\dd$ & Av. $\lm$ (stnd. err.) \\  \hline
RAND & 10,000 & 28 & 0.919 (0.002) \\ \hline
MNIST & 5,000 & 784 & 0.865 (0.005) \\ \hline
YALEBXF & 2,000 & 32,256 & 0.963 (0.008)  \\\hline
RCV1 & 4,000 & 29,992 & 0.485 (0.246)  \\\hline
COIL & 6,000 & 49,152 & 0.981 (0.019)  \\\hline
GTZAN & 12,000 & 199 & 0.988 (0.009)  \\\hline
NYT & 299,000 & 102,660 & 0.714 (-)  \\\hline
    \end{tabular}
 \caption{\small Summary of the datasets.
 The reported value of
 $\lm$ is obtained by averaged over the lasso instances solved.}
  \label{tab:data_sets}
 \vspace{-6mm}
\end{table}

\section{Numerical Examples}\label{sec:exp}
We now examine the performance of the 
screening algorithms presented using the datasets 
summarized in Table \ref{tab:data_sets}, and discussed 
in detail below.

\noindent
\textbf{(1) RAND:}
We generate lasso problems with $\dd=28$ and $\ncw=10,000$ 
by randomly generating  $10,001$ $28$-dimensional vectors
$\tv,\vb_1,\ldots,\vb_{10,000}$. 
These vectors are scaled to unit norm.\\
\textbf{(2) MNIST}:
$70,000$  images ($28\times 28$) of hand-written digits
($60,000$ and $10,000$ in the training and testing sets, respectively)
\cite{LeCun1998The-MNIST, Lecun1998Gradient-based}.
We form a dictionary by randomly sampling $500$ training images 
for each digit, and a target vector from the testing set.
Each image is vectorized and scaled to unit norm.\\
\textbf{(3) YALEBXF}:
Frontal face images ($192\times 168$) of
38 subjects in the extended Yale B face dataset \cite{Georghiades2002From, Lee2005Acquiring}. We randomly select $\ncw=2,000$ of the $2,414$ images as the dictionary, and $\tv$ from the remaining $414$ images. Each image is vectorized and scaled to unit norm. \\
\textbf{(4) RCV1}: 
A bag-of-words representation of four classes from the Reuters Corpus Volume 1 (RCV1) dataset \cite{RCV1}. There are 9,625 documents with 29,992 distinct words, including categories ``C15'', ``ECAT'', ``GCAT'', and ``MCAT'', each with 2,022, 2,064, 2,901, and 2,638 documents respectively.
The vector representations have an average of $75.9 \pm 60.0$ nonzero entries; a sparsity of $0.25\% \pm 0.19\%$.\\
\textbf{(5) COIL}: 
Images ($128 \times 128\times3$) 
of 100 objects, with 72 images per object obtained by rotating the object every 5 degrees \cite{Nene1996Columbia}. \\
\textbf{(6) GTZAN}: 100 music clips (30 sec, sampled at 22,050 Hz) for each of ten genres of music \cite{Tzanetakis2002}.
Each clip is divided into 3-sec adjacent texture windows (TW) with
50\% overlap. Each TW is represented using a first order scattering vector of length 199 \cite{Mallat2011}. \\
\textbf{(7) NYT}: 
A bag-of-words dataset 
in which 300,000 New York Times articles are represented as vectors with respect to a vocabulary of 102,660 words  \cite{Frank+Asuncion:2010}. The $i$-th entry in vector $j$ gives the number of occurrences of word $i$ in document $j$. Documents with low word counts are removed, leaving 299,752 documents.

\begin{figure*}[t!]
\centerline{
\includegraphics[width=0.6\textwidth]{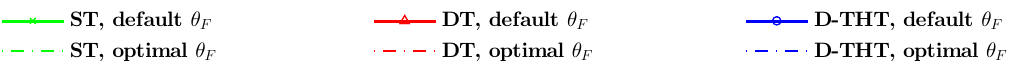}
}
\centerline{
\includegraphics[width=0.22\textwidth]{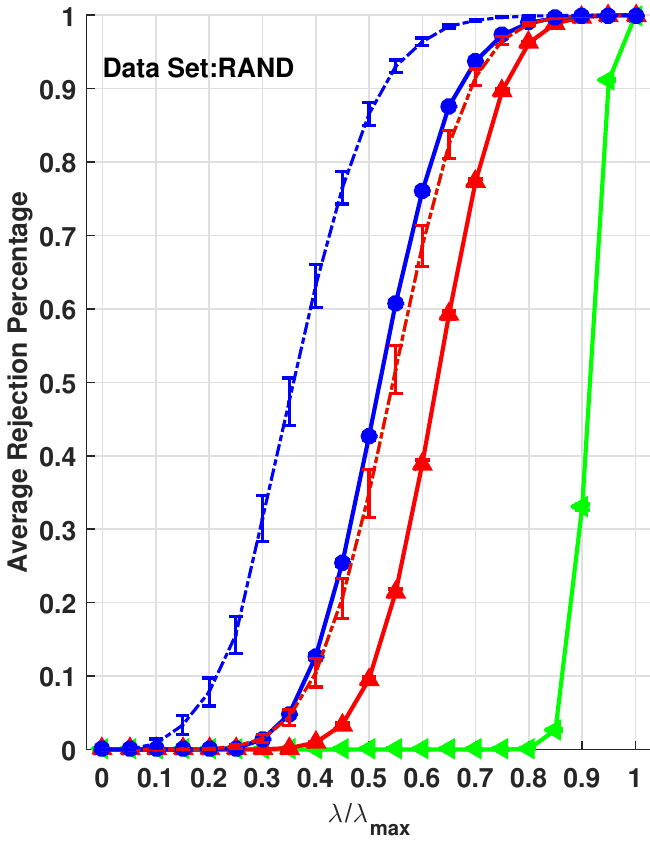}
\includegraphics[width=0.22\textwidth]{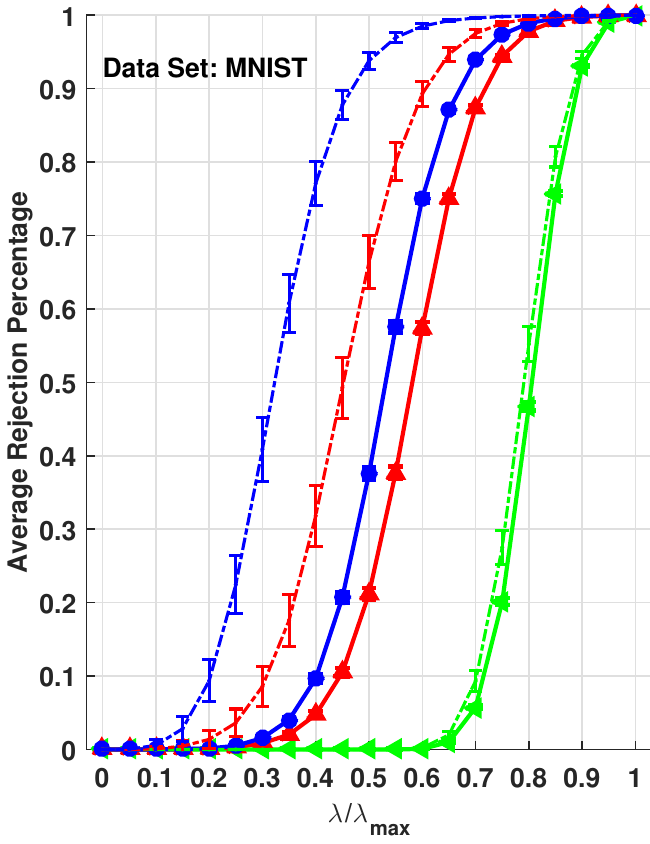}
\includegraphics[width=0.22\textwidth]{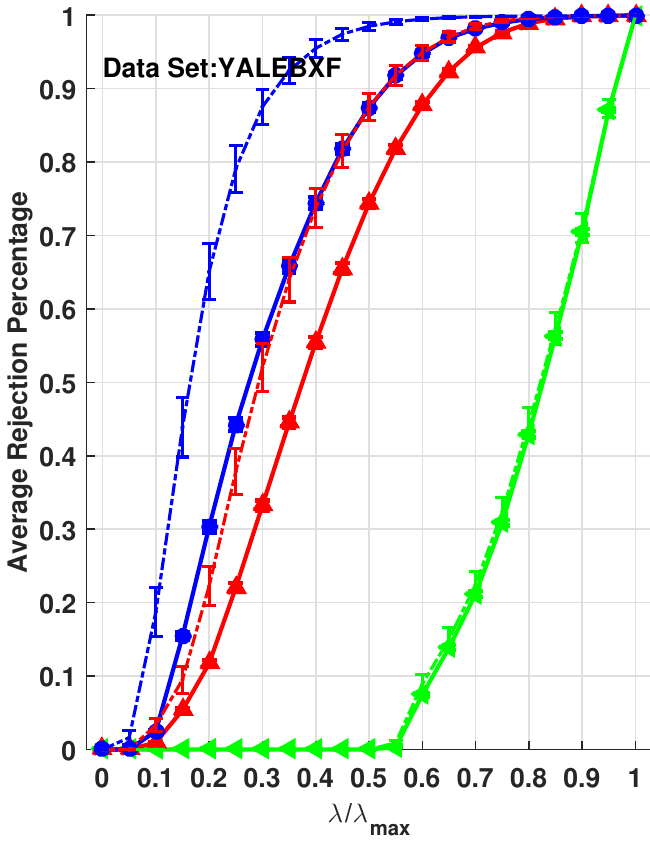}
}
\centerline{
\includegraphics[width=0.22\textwidth]{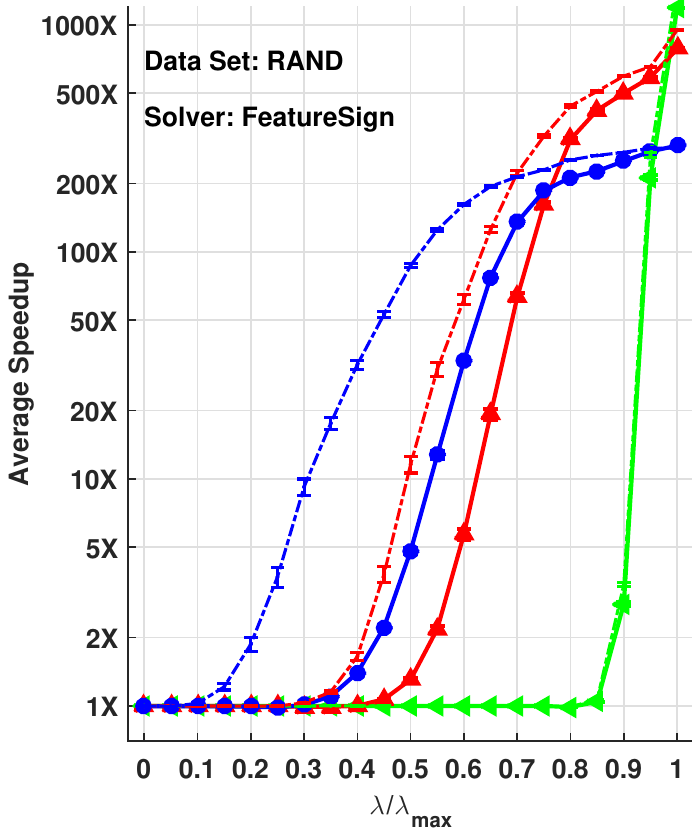}
\includegraphics[width=0.22\textwidth]{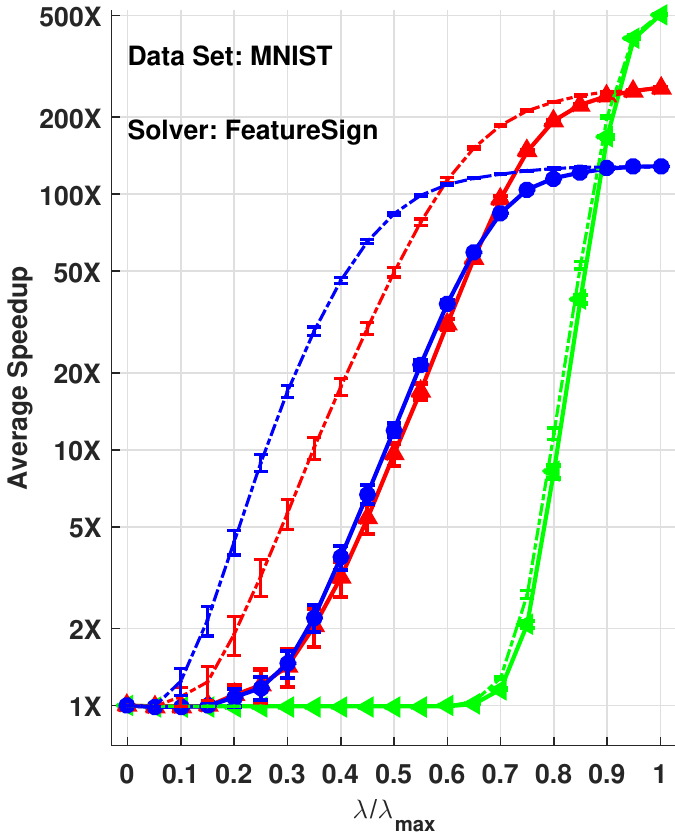}
\includegraphics[width=0.22\textwidth]{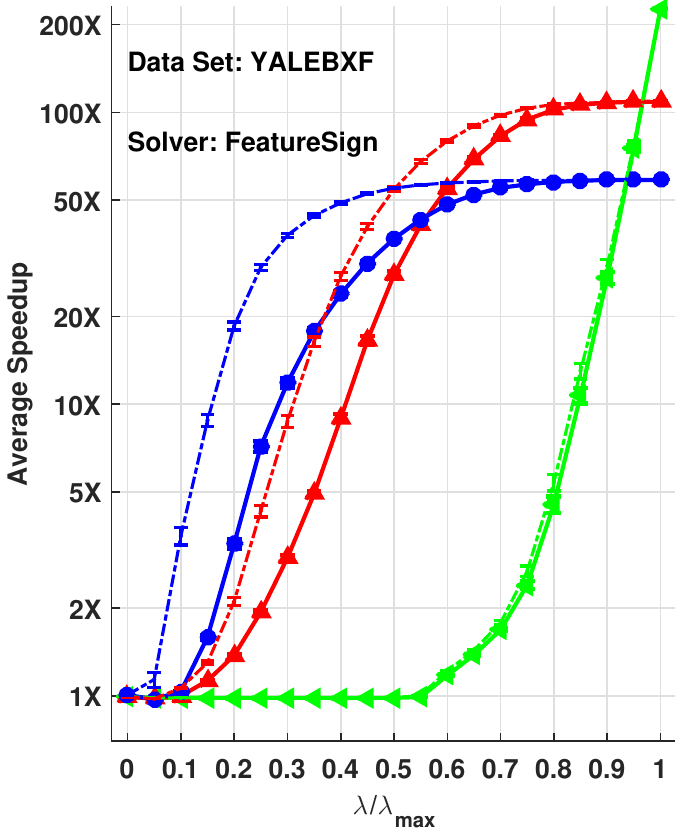}
}
\caption{ {\small
Performance of ST, DT and D-THT.
Top: rejection percentage;
Bottom: speedup using screening and the FeatureSign solver \cite{Lee2007Efficient}. Solid curves lower bound and dashed curves upper bound performance for spherical bounds centered at $\tv/\lambda$.} } \label{fig:C_THT}
\end{figure*}

All experiments solve the standard lasso problem \eqref{eq:iterW} using the Feature-sign \cite{Lee2007Efficient} and FISTA \cite{AM2009} solvers.
The grafting solver \cite{Perkins2003Online} was also tested and gave similar qualitative performance.
We use two performance metrics: the percentage of features rejected and the speedup 
(time to solve the lasso problem divided by sum of the time to screen \emph{and} the time to solve the reduced lasso problem).
Timing and speedup results depend on the solver used.
The regularization parameter $\lambda$ is set using the scaling invariant ratio $\fracn{\lambda}{\lm}$ where $\lm=\max_i |\tv^T\vb_i|$.
So $\fracn{\lambda}{\lm} \in [0,1]$ with larger values yielding sparser solutions.
For all datasets except RCV1 and NYT, 
we randomly select 20 dictionaries 
and for each dictionary we use 60 randomly selected test vectors. Averaged metrics and standard errors are reported across these $1200$ lasso instances. 
For RCV1, since $\lm$ is very low, we select $496$ lasso instances with $\lm \geq 0.5$ from the pool of $1200$ instances and report results for these $496$ instances.
For the very large NYT dataset, we select the first 299,000 examples as the dictionary and 6 documents from the remaining 752 as target vectors. 
\begin{figure*}
\centerline{
\includegraphics[width=0.22\textwidth]{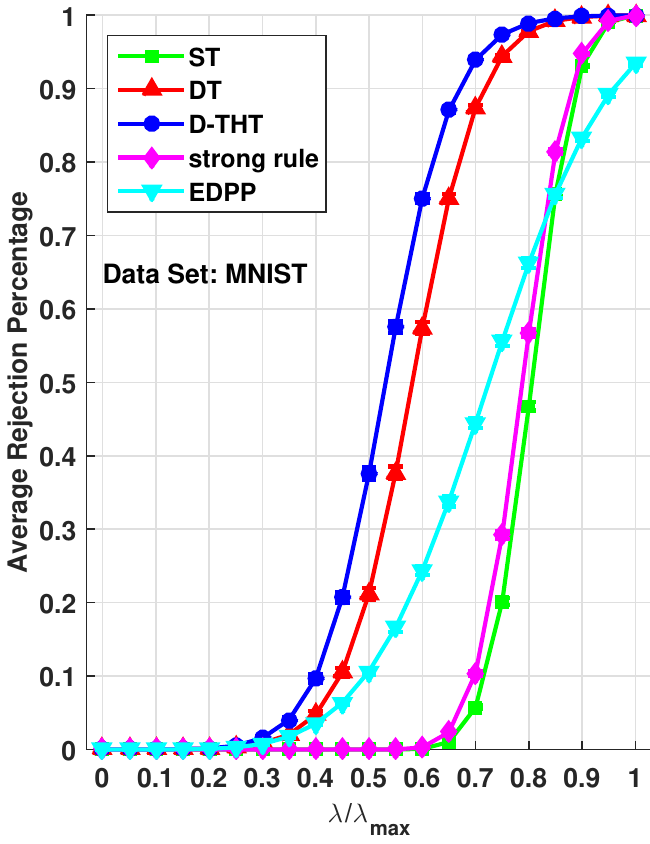}
\includegraphics[width=0.22\textwidth]{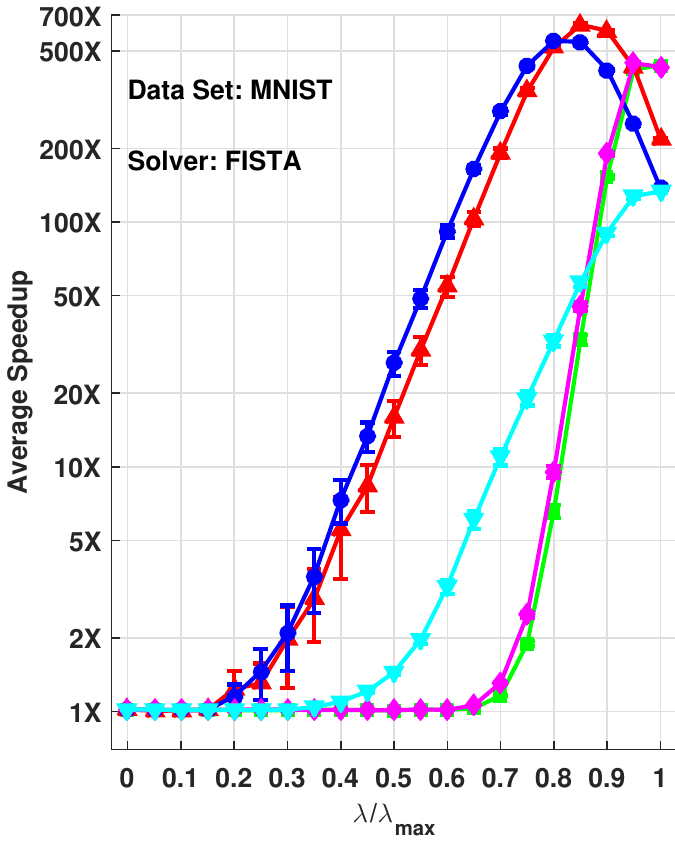}
\includegraphics[width=0.22\textwidth]{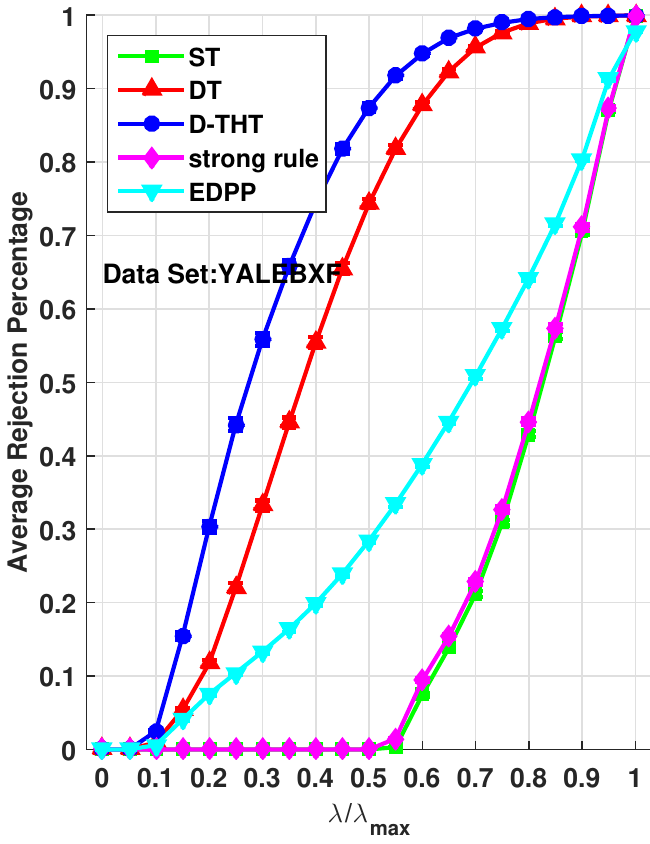}
\includegraphics[width=0.22\textwidth]{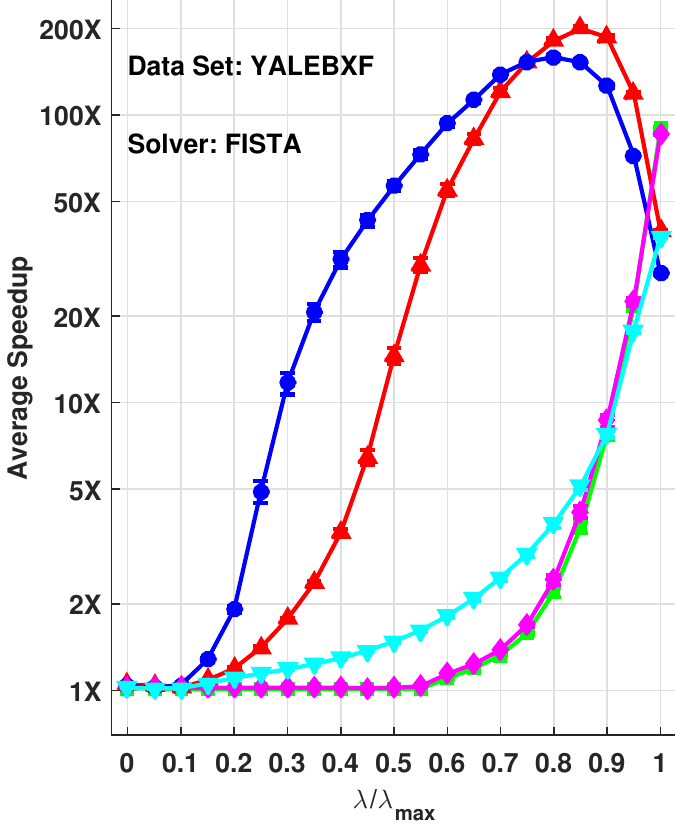}
}
\caption{\small
Performance comparison of ST, DT, D-THT (all with default $\vtf$), enhanced DPP (EDPP) \cite{dpp2015} and the strong rule \cite{Tibshirani2010Strong}
using the FISTA solver on the MNIST and YALEBXF datasets.}
\label{fig:one_shot}
\end{figure*}

\begin{figure*}
\centerline{
\includegraphics[width=0.22\textwidth]{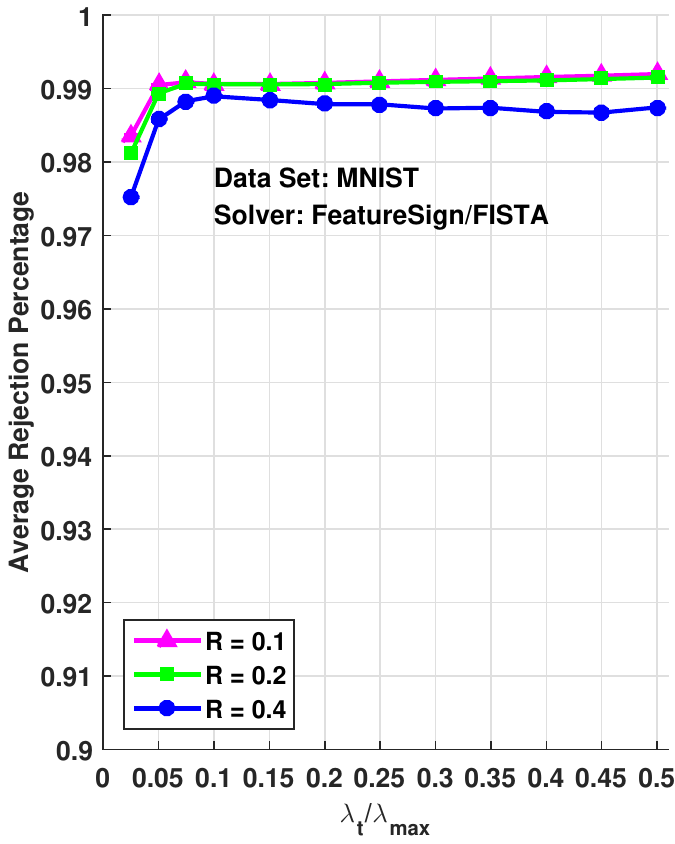}
\includegraphics[width=0.22\textwidth]{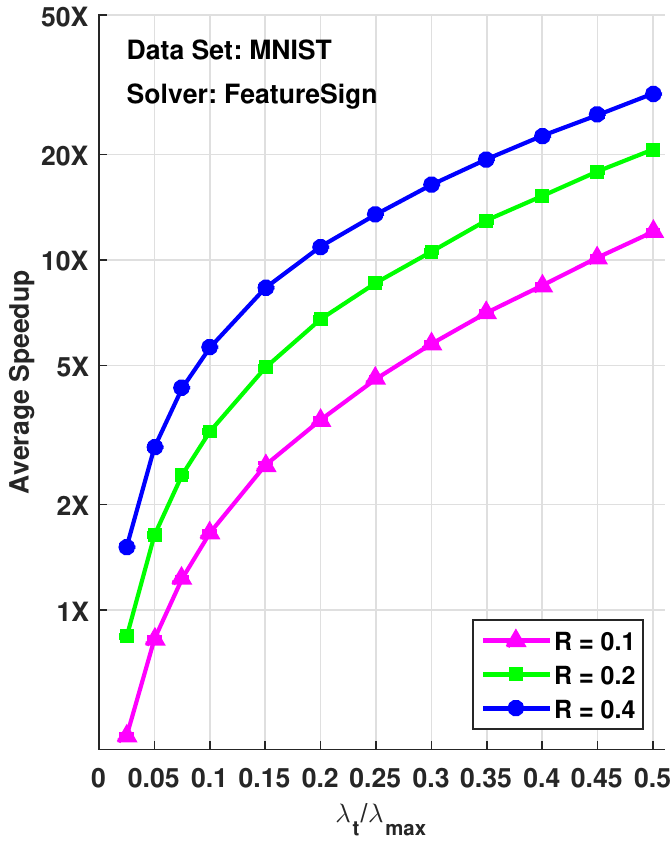}
\includegraphics[width=0.22\textwidth]{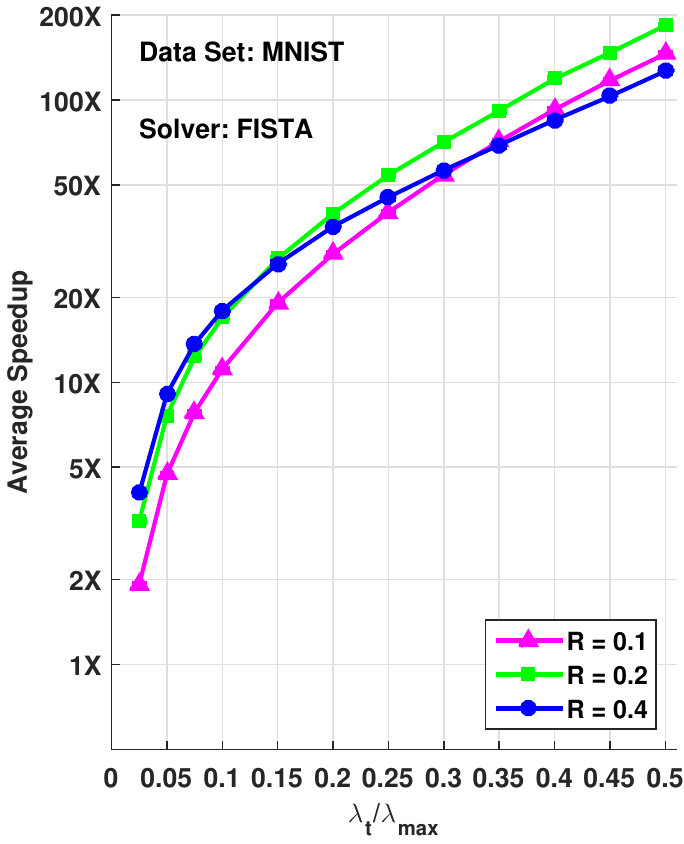}
\includegraphics[width=0.22\textwidth]{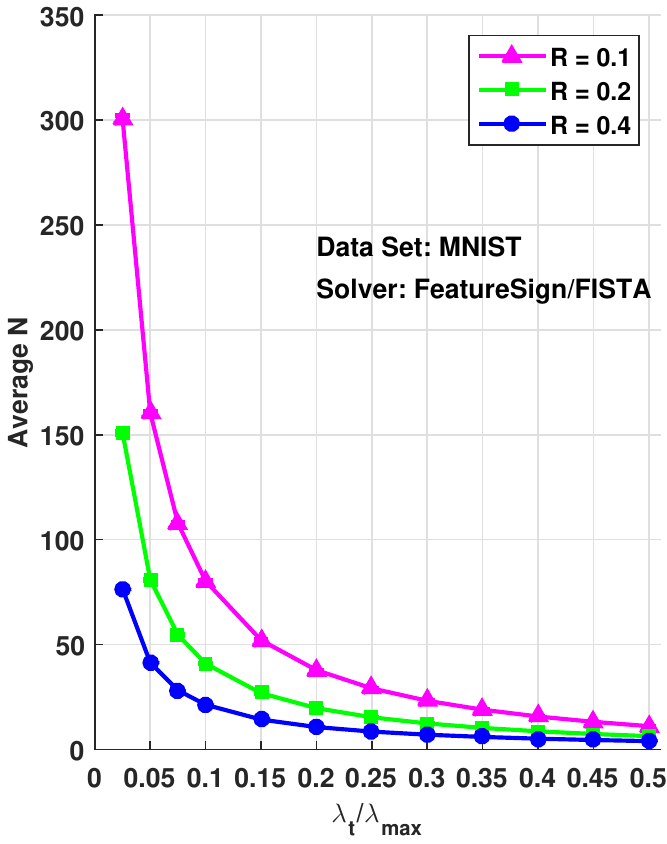}
}
\centerline{
\includegraphics[width=0.22\textwidth]{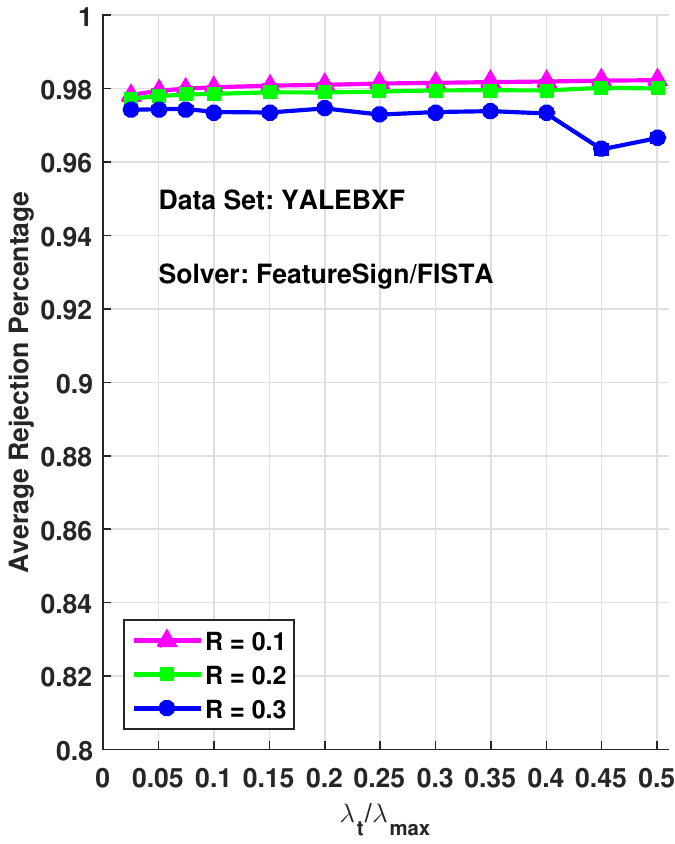}
\includegraphics[width=0.22\textwidth]{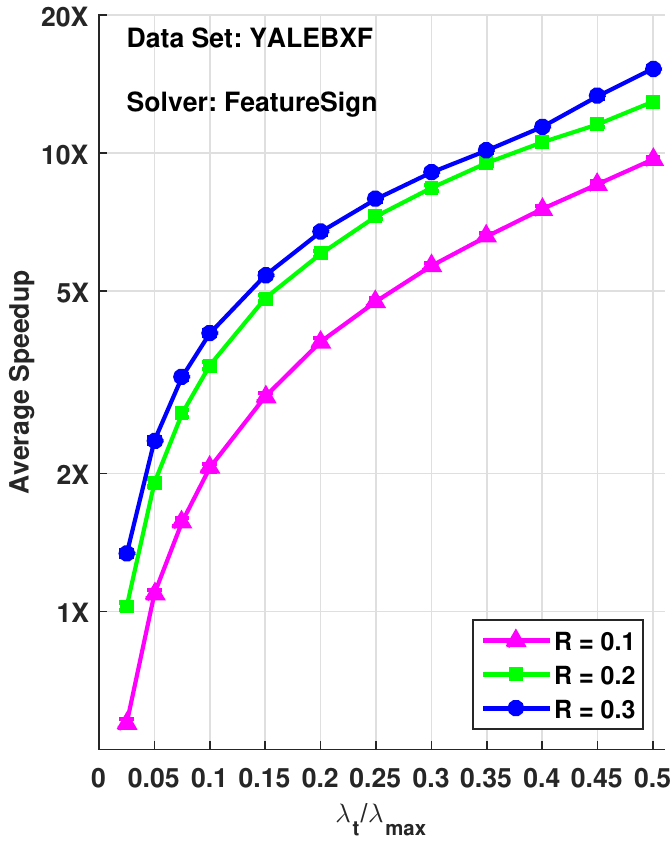}
\includegraphics[width=0.22\textwidth]{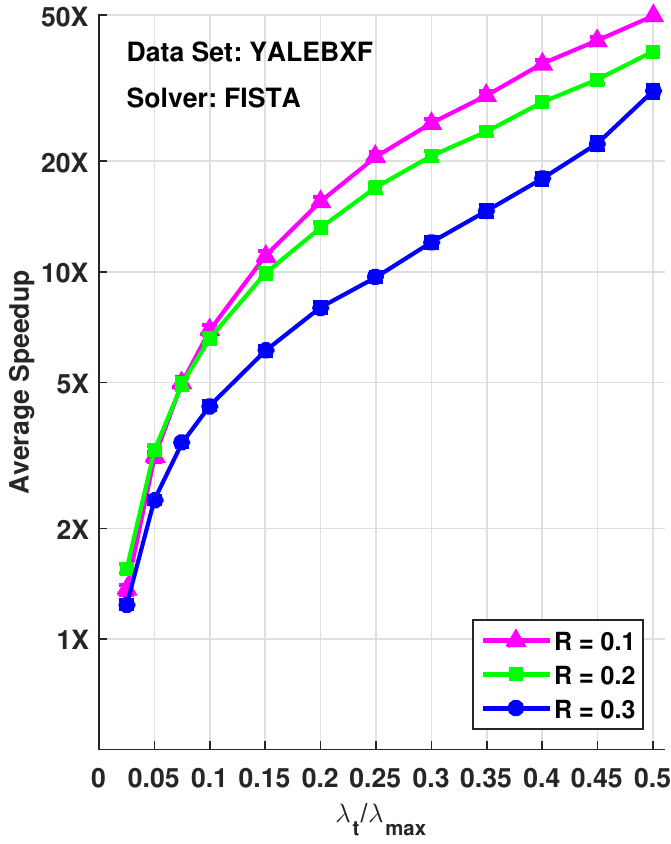}
\includegraphics[width=0.22\textwidth]{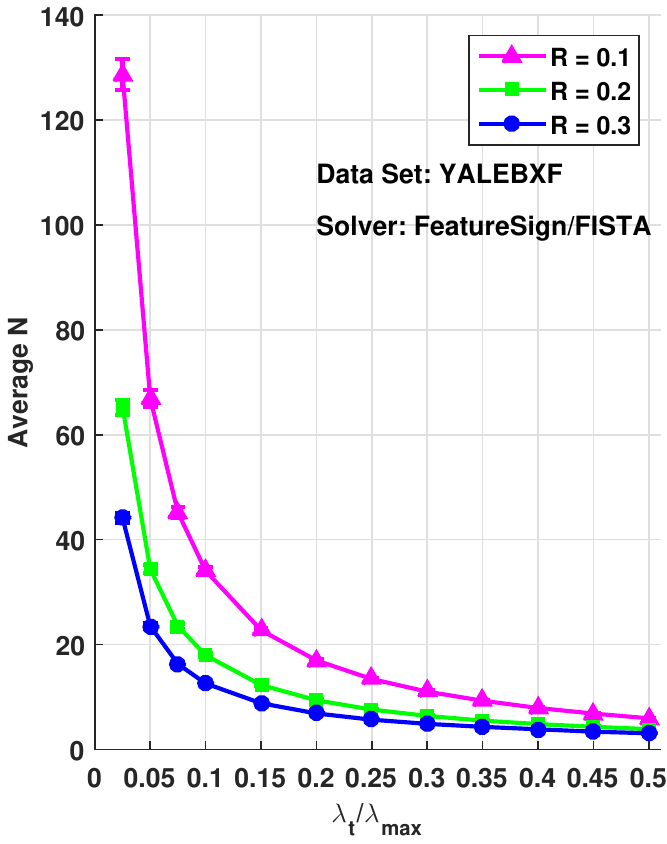}
}
\caption{\small
Data-Adaptive Sequential Screening (DASS) applied to MNIST (top) and YALEBXF (bottom) using the feature-sign and FISTA solvers.
(Left): average rejection percentage.
(Middle): Speedup factor.
(Right): The average value of $N$.}
\label{fig:fixedR}
\end{figure*}

\subsection{The performance of one-shot screening}
We first benchmark the performance of the one-shot tests:
ST (\S\ref{sec:st}), DT (\S\ref{sec:dt}), and D-THT (\S\ref{sec:THT}).
We first use the default spherical bound \eqref{eq:sb1}.
This gives a lower bound for the performance of the one-shot screening methods on each dataset.
The default dome test combines this sphere with the feature $\atm$, while dictionary-based THT combines it with two features using the selection scheme detailed in \eqref{eq:firstb}, \eqref{eq:scndb}.
We also show results using a second ``oracle'' bounding sphere with center $\tv/\lambda$ and radius  $r=\|\tv/\lambda-\vto\|_2$. This provides an upper bound on performance over bounding spheres centered at $\tv/\lambda$.

The performance of the one-shot screening methods on the test datasets 
based on the feature-sign solver \cite{Lee2007Efficient} are shown in Fig.~\ref{fig:C_THT}. 
Here are the salient points:
(a) While the default one-shot tests perform well for high values of $\lambda/\lm$, this performance quickly degrades as $\lambda/\lm$ decreases. At values of $\lambda/\lm$ around 0.2 and lower, the tests are not effective.
(b) On the other hand, the upper bounds indicate potential for improvement if a better spherical bound can be found. Indeed, the significant gap between the lower and upper performance bounds suggests that it is worth investing computation to improve the default spherical bound.
(c) Among the tested methods, D-THT exhibits the best performance except at very high values of $\lambda/\lm$. On RAND, for example, using
$\lambda/\lm=0.5$ and the default spherical bound, D-THT yields a 400\% rejection improvement over DT. The concurrent speedup for D-THT is about 5X while for DT is less than 2X. These effects are also seen for MNIST and YALEBXF.

Fig.~\ref{fig:one_shot} shows a performance comparison between 
ST, DT, D-THT (all with default $\vtf$), EDPP \cite{dpp2015} and the strong rule \cite{Tibshirani2010Strong} 
using the FISTA solver \cite{AM2009}.
Here are the salient points:
(a) Aside from the small dip at high values of $\lambda/\lm$, the speedup trend for the FISTA solver is similar to that for feature-sign. For the datasets we tested, feature-sign seems to be faster than FISTA, 
but FISTA is more sensitive to the reduction in dictionary size 
resulting from screening. 
Thus it has greater speedup. This can also been seen in Fig.~\ref{fig:fixedR}.
(b) Of the one-shot methods tested, dictionary based THT and DT consistently have the best rejection performance. But while current one-shot screening tests can perform well at moderate to high values of $\lambda/\lm$, such performance does not extend to the important range of low values of $\lambda/\lm$.

The rejection and speedup of IRDT (not plotted)  and D-THT were very similar on the test datasets with IRDT terminating after 3 or 4 iterations at the break in line 14-16 in Algorithm \ref{alg:IRDT}. 

\subsection{The performance of sequential screening}
To explore the effectiveness of sequential screening, we tested the Data-Adaptive Sequential Screening (DASS) scheme \eqref{eq:updater}.
The performance results are shown in Fig.~\ref{fig:fixedR}.
Here are the salient points:
(a) For both MNIST and YALEBXF, with $R=0.2$ the performance of DASS is robust across a variety of values of $\lat$;
(b) DASS yields significant improvement in rejection fraction and robust speedup performance compared with one-shot tests;
(c) At values of $\lambda/\lm$ around 0.1 and lower,
DASS is rejecting 98\% of the dictionary while giving speedup greater than 1.
This is successful screening at much lower values of $\lambda/\lm$.

\begin{figure}[t!]
\centering
\includegraphics[width=0.15\textwidth]{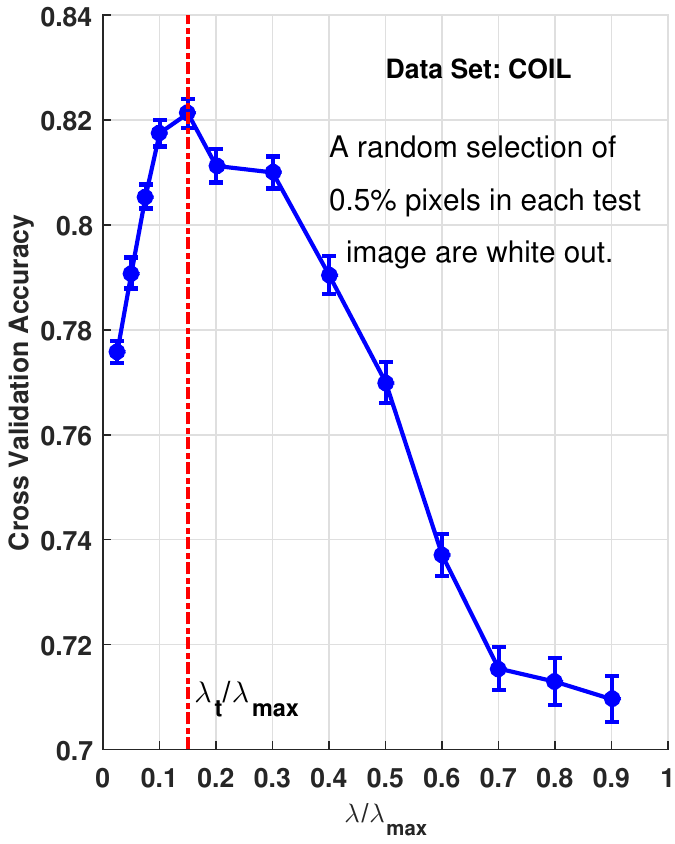}
\includegraphics[width=0.15\textwidth]{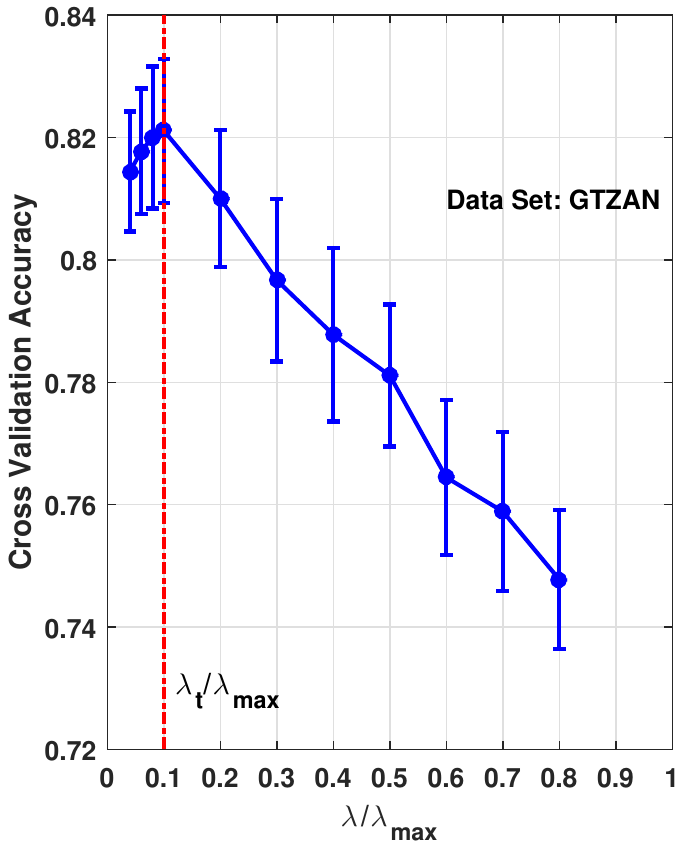}
\includegraphics[width=0.15\textwidth]{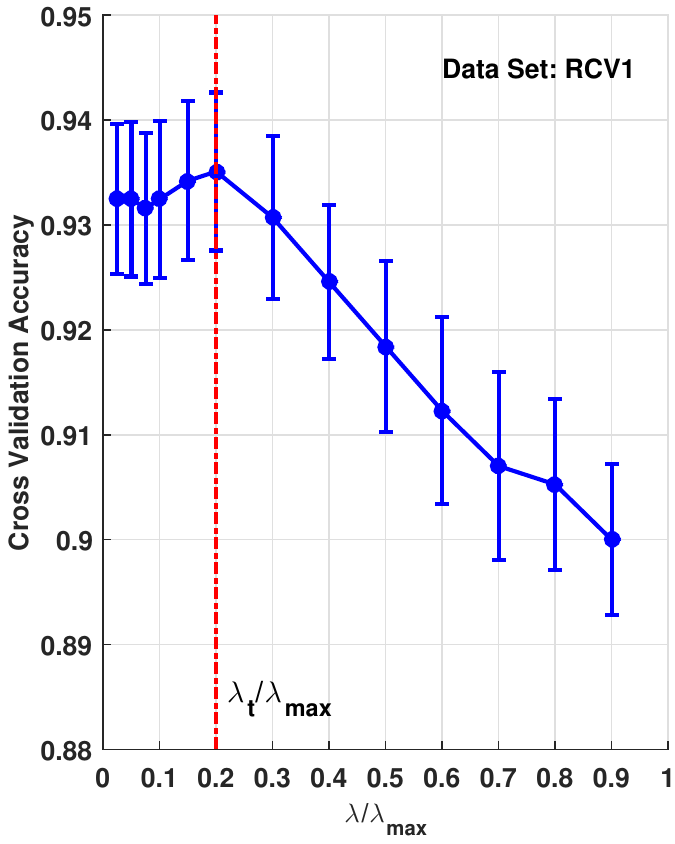}
\includegraphics[width=0.44\textwidth]{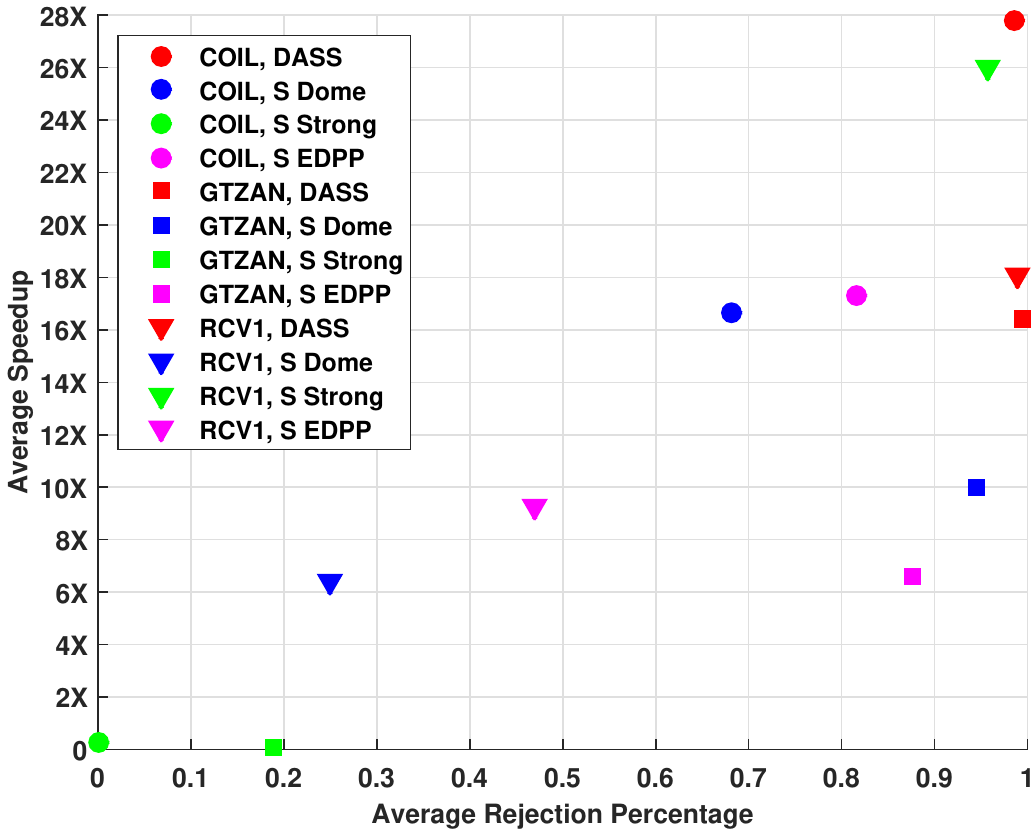}

\caption{\small
Performance of four sequential screening algorithms (DASS, sequential dome, sequential strong rule, sequential EDDP) for the screening of the lasso problems in SRC on three datasets (COIL, GTZAN, RCV1).
(Top): cross validated accuracy to determine the best $\lat/\lm$.
(Bottom): Speedup vs Rejection at the best $\lat/\lm$ 
for each dataset.
}\label{fig:application_impact}
\vspace{-5mm}
\end{figure}

\subsection{Sequential screening and classification}
Now we focus on specific values of $\lambda/\lm$ motivated by practical  lasso problems and examine how screening can help. To do so, we 
use the COIL, GTZAN and RCV1 datasets to examine the impact of 
sequential screening in Sparse Representation Classification (SRC) \cite{Wright2009Robust}. 
Although SRC was first proposed for face recognition problems, it is a generic multi-class classifier that has found success in a variety of applications. The time and memory consuming step in SRC is solving a lasso problem. 
For the COIL dataset we made the SRC problem more challenging by saturating a random subset  of 0.5\% of the pixels to white. 

We first use cross-validated prediction accuracy to determine the best values for $\lat/\lm$ for SRC when applied to  the  datasets. The results  (top row of Fig.~\ref{fig:application_impact}) are  COIL: $\lat/\lm = 0.15$, 
GTZAN: $\lat/\lm = 0.1$, and
RCV1: $\lat/\lm = 0.2$.
For these specific values of $\lat/\lm$, we then examine the performance of the following screening schemes in solving SRC problems for these 
datasets:
(1) the feedback scheme DASS, and the open loop sequential screening schemes (2) sequential dome test, 
(3) sequential strong rule \cite{Tibshirani2010Strong} 
and (4) sequential EDPP rule \cite{dpp2015}. 
We select the parameters of each method to keep the average value of $N$  the same. Since DASS uses a variable value of $N$, we first select its parameters, then use the resulting average value of $N$ for the open loop schemes.
For COIL, DASS with $R=0.5$ yields an average $N = 4.72$;
for GTZAN, DASS with $R=0.15$ yields an average $N = 14.63$; and
for RCV1, DASS with $R=1$ yields an average $N = 3.59$. 
Then for the open loop sequential screening schemes we set $N = 5$ for COIL, $N = 15$ for GTZAN  and $N = 4$ for RCV1. 
This keeps the average path lengths of the screening schemes the same.
The results are shown in the bottom row of Fig. ~\ref{fig:application_impact}. 

Here are the salient points:
(a) Over $50\%$ of the experiments (dataset+screening method) gave a speedup of at least $10X$. So sequential screening offers considerable potential gain in practical applications.
(b) At the high end, DASS provided $28X$, $16X$ and $18X$ speedup in solving SRC lasso problems for the three datasets. 
That's an average speedup of $21X$.
However, given that we only used three datasets and did not ``tweak'' each method to find its best performance on each dataset, we can't conclude that one method is always better than the rest. That would require a more extensive investigation.
Finally, although the strong rule can't rule out false rejections, 
we detected no false rejections in our experiment.

\subsection{Sequential screening on a large dataset}
Finally, we used the NYT dataset to explore how successfully one can screen and solve lasso problems using small values of $\lambda/\lm$ with high dimensional data and a very large dictionary. We normalize each document and randomly selected six documents 
from the first 100 of the 752 held out documents subject to 
$0.5 <\lm <0.9$.
DASS Screening (with
$\lat/\lm=0.1$, $\lambda_1=0.95\lm$ and $R=0.3$) was done in an ``on-line'' mode by loading only small amounts of the dictionary into memory at a time.
The value of $N$ is selected automatically for each instance.
In all tested cases, $N\leq 27$.
As a benchmark, we tested a geometrically spaced, open loop sequential screening algorithm (sequential THT) using $\lat/\lm=0.1$, $\lambda_1=0.95\lm$ and $N=30$.

The results for both methods are shown in Fig.~\ref{fig:NYT}.
We can't solve these lasso problems without using screening. 
Hence the usual speedup metric can't be evaluated.
The main time cost is sequentially reading features from disk into RAM.
Here are the main points to note:
(a) Under geometric spacing with fixed $N$, 
less than 10,000 of the features (3.3\%)
were held in memory at once;
(b) For DASS, less than 1,000 of the features (0.33\%) were held in memory at once -- an order of magnitude improvement over fixed geometric spacing
(The small dip at $\lat$ is due to termination method);
(c) On this dataset, both open loop sequential screening and DASS clearly exhibit a  significant performance advantage over one-shot tests. The use of feedback by DASS to automatically select the number of steps $N$ and the values $\{\lambda_k\}_{k=1}^N$, yields robust rejection performance.
By tweaking $N$ for each test vector in the open loop scheme, one could improve its average performance. But DASS handles this automatically
and robustly.
\begin{figure}[t!]
\centering
\includegraphics[width=0.42\textwidth]{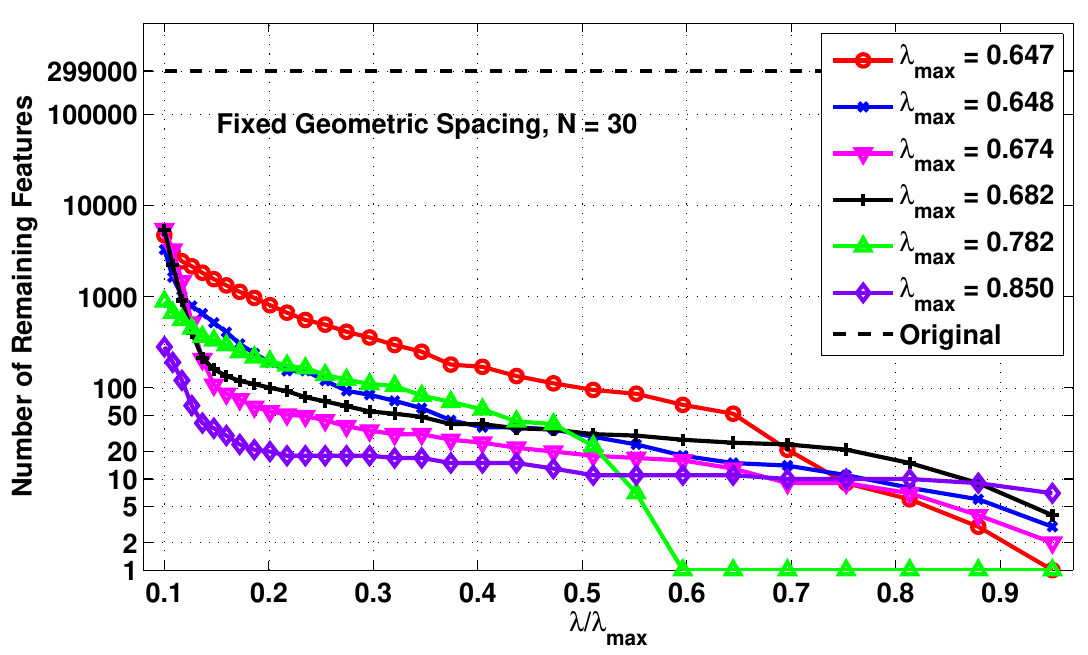}\\
\includegraphics[width=0.42\textwidth]{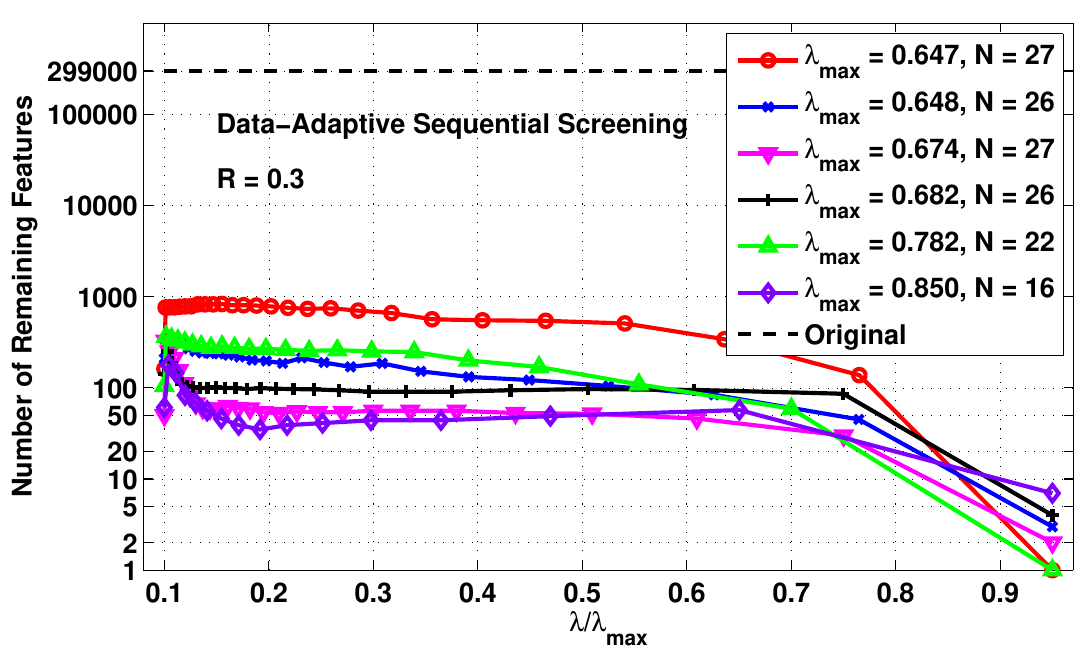}
\caption{\small
Sequential screening using on the NYT dataset.
(Top): Open loop geometric spacing using $N=30$ and THT test;
(Bottom): DASS with $R=0.3$.
For six problems $(\tv_i,\lat)$, $i=1,\dots,6$ with $\lat=0.1\lm$, the plotted points indicate the number of surviving features after THT screening of each instance $(\tv_i,\lambda_{ik})$, $k=1,\dots,N$, $i=1,\dots, 6$.
}\label{fig:NYT}
\vspace{-5mm}
\end{figure}

\section{Discussion and Conclusion}\label{sec:discuss}
In our survey we have emphasized separating the discussion of test structure from the problem of selecting its parameters. This allowed us to see connections between many existing screening tests, and enabled a clearer understanding  of screening in general. Hopefully this will be advantageous to the development of new tests and parameter selection methods.

For one-shot screening tests,
our numerical studies on THT strongly suggest that more complex region tests are indeed worthwhile. THT gave significant performance improvement beyond simpler tests in both rejection and speedup over important ranges of $\lambda/\lm$ values. But
the performance of one-shot tests is still inadequate at 
small values of $\lambda/\lm$.
The numerical studies also indicated a significant performance gap between using the default spherical bound and the best bound at the same sphere center. This
indicates the value of additional computation to improve the spherical bound.

Our empirical studies have shown that sequential screening (for example, DASS) can significantly extend useful screening performance to a wider range of $\lambda/\lm$.
DASS has the additional advantage that it selects both the number $N$ and the sequence $\{\lambda_k\}_{k=1}^N$ automatically.

Screening is critical when the dictionary will not fit into available memory. We have demonstrated a successful application of DASS to a very large NYT dataset, of dimension $102,660$ by $299,000$. To the best of the authors' knowledge, with constrained computational resources, screening is the only way to solve lasso problems of this size.

The concepts described in this survey should provide a firm foundation for understanding screening for related sparse representation problems.
This includes screening for the elastic net (reducible to lasso problem),
$\ell_1$ regularized logistic regression,
the graphical lasso,
and the group lasso \cite{Tibshirani2010Strong}.
In addition, SAFE methods have been developed for the sparse support vector machine and logistic regression in \cite{Ghaoui2012}, and the group lasso in \cite{dpp2015}. Recently, Liu et al. \cite{variational2013} have proposed safe screening for generalized sparse linear models. This makes use of the variational inequality that provides a necessary and  sufficient optimality condition for the dual problem.
Dash et al., \cite{DashICASSP2014}, consider screening for Boolean compressed sensing in which the objective is to select a sparse set of Boolean rules that are predictive of future outcomes. One of the screening rules developed is based on the duality arguments presented here. Targeting problems that use nuclear norm regularization to pursue a low rank matrix solution, Zhou et al., \cite{SSS}, have recently proposed safe subspace screening for nuclear norm regularized least squares problems.
Wang et al., \cite{SRCvSLS2013},  have integrated DASS with sparse representation classification, to speed up classification, and
Jao et al., \cite{JaoICASSP2014}, have applied screening to the problem of representing music in terms of an audio codebook (dictionary) for genre tagging.
We expect to see more such applications as the size of dictionaries increase.

\appendices

\section{Proofs \S\ref{sec:prelim}}\label{sec:proof_prelim}
The dual lasso problem \eqref{eq:dual} is obtained as follows.
Setting  $\vz=\tv-\Dict\wv$ in \eqref{eq:iterW} gives the constrained problem:
$\min_{\vz, \wv} \half ~\vz^{T}\vz+\lambda\|\wv\|_{1}$, subject to $\vz=\tv-\Dict\wv$.
Form the Lagrangian
$\mathcal{L}(\vz,\wv,\vmu)
= \half~\vz^{T}\vz+\lambda\|\wv\|_{1}
+\vmu^{T}(\tv-\Dict\wv-\vz)$
and compute the subdifferentials
with respect to $\vz$ and $\wv$. Using
the condition that $\vzero$ must be in each subdifferential gives $\vmu=\vzo$
and the constraints $|\vmu^{T}\at_{i} |\leq \lambda$, $i=1,\dots,\ncw$.
The above equations allow the elimination
of $\vz$ and $\wv$ from $\mathcal{L}$.
This leads to the dual problem:
$\max_{\vmu} \half~\|\tv\|^{2}_{2}
-\half \|\vmu-\tv\|_{2}^{2}$,
subject to $|\vmu^{T}\at_{i} |\leq \lambda$, $i=1,\dots,\ncw$.
The change of variable $\vt=\vmu/\lambda$ then gives \eqref{eq:dual}.
By construction, the primal and dual solutions $\wvo$ and $\vto$
are related through \eqref{eq:relationship}.

\section{Proofs \S\ref{sec:screen}}\label{sec:prscreen}

\begin{proof}[Theorem \ref{thm:crt}]
The proof for the lasso is given above the theorem statement.
For the nn-lasso, only the definition of the active set changes.
\end{proof}

\begin{proof}[Corollary \ref{cor:compliantS}]
In the proof of Theorem \ref{thm:crt}, the
inclusion $\bar S\subseteq \bar A(\vto)$, gives
$i\in \bar S$ implies $\weo_{i}=0$.
\end{proof}

\section{Proofs \S\ref{sec:region}} \label{sec:prregion}

\begin{proof}[Corollary \ref{cor:comregiontest}]
This is proved above the corollary.
\end{proof}

\begin{proof}[Lemma \ref{lem:subset}]
	Assume $\RR_1 \subseteq \RR_2$.
	If $\RR_{1}=\emptyset$, the result is clear.
	Hence assume $\RR_{1}\neq\emptyset$.
	Note that $\mu_{\RR_1}(\at)
	=\max_{\vt\in \RR_{1}}\vt^{T}\at \leq
	\max_{\vt\in \RR_{2}}\vt^{T}\at
	=\mu_{\RR_2}(\at)$.
	For the lasso, if $\at_i$ is rejected
	by $T_{\RR_2}$,
	then $\mu_{\RR_2}(\at_{i})<1$ and
	$\mu_{\RR_2}(-\at_{i})<1$.
	Hence  $\mu_{\RR_1}(\at_{i})<1$ and
	$\mu_{\RR_1}(-\at_{i})<1$.
	So $\at_i$ is also rejected by
	$T_{\RR_1}$. Therefore $T_{\RR_2}
	\comp T_{\RR_1}$. The proof for the
	nn-lasso is similar.
\end{proof}

\section{Proofs \S\ref{sec:st}}\label{sec:prst}

\begin{proof}[Lemma \ref{lem:maxStb}]
	By Cauchy-Schwarz:
	$	\vt^{T} \at
	=  (\vt-\vq)^{T}\at + \vq^{T}\at
	\leq \norm{\vt-\vq}{2}\norm{\at}{2}
	+ \vq^{T}\at$
	with equality when $\vt-\vq$
	is aligned with $\at$.
	Then $\vt\in \SPH(\vq,r)$ ensures
	$\vt^{T}\at \leq  r\|\at\|_{2} +  \vq^{T}\at$
	with equality when $\|\vt-\vq\|_2=r$.
\end{proof}

\begin{proof}[Theorem \ref{thm:st}]
	For the nn-lasso we reject $\at_i$ if
	$\mu_{\SPH}(\at_{i}) <1$.
	So \eqref{eq:st} follows from \eqref{eq:musph}.
	For the lasso we reject $\at_i$ if
	$\mu_{\SPH}(\at_{i}) <1$ and
	$\mu_{\SPH}(-\at_{i}) <1$,
	i.e., if
	$\vq^{T}\at_{i}<(1-r\|\at_i\|_2)$ and
	$\vq^{T}\at_{i}>-(1-r\|\at_i\|_2)$.
	This gives \eqref{eq:st}.
	Note
	$\max\{\mu_{\SPH}(\at_i), \mu_{\SPH}(-\at_i)\}<1$
	$\Leftrightarrow$
	$\max \{\vq^T\at_i + r\|\at_i\|_2,
	-\vq^T\at_i +r\|\at_i\|_2\}<1$
	$\Leftrightarrow$
	$|\vq^T\at_i| < 1-r\|\at_i\|_2$.
	Thus \eqref{eq:st}
	and \eqref{eq:stavf} are equivalent.
\end{proof}

\section{Proofs \S\ref{sec:dt}}\label{sec:prdt}

\begin{proof}[Lemma \ref{lem:max-thTb}]
Solving \eqref{eq:optprobm} with $m=1$ is equivalent to solving the Lagrangian problem:
\begin{align}
\begin{split} \label{eq:DM_optim_dual}
&\max_{\mu, \sigma \geq 0} \ \min_{\vt}\ \
\mathcal{L}(\vt, \mu, \sigma) = -\vt^T\at\\
& \quad	+ \mu [(\vt-\vq)^T(\vt-\vq)-r^2]+\sigma (\vn^T\vt-c) .
\end{split}
\end{align}
Setting the derivative w.r.t. $\vt$ equal to zero yields
$\vt=\vq+\fracn{\at}{2\mu}-\fracn{\sigma\vn}{2\mu}$.
Substituting $\vt$ into $\mu_{\DM}$ and \eqref{eq:DM_optim_dual}:
\begin{align}
\mu_{\DM}(\at)
&=\at^T\vt=\at^T\vq+\fracn{\|\at\|_2^2}{2\mu}-\fracn{\sigma\at^T\vn}{2\mu}\label{eq:DMmu},\\
\mathcal{L}(\mu,\sigma)
&= \textstyle -\vq^T\at-\frac{\|\at\|^2}{4\mu}-\frac{\sigma^2}{4\mu}-\mu r^2+\sigma r \fr +\frac{\sigma t}{2\mu}, \label{eq:DMLms}
\end{align}
where
$\fr= \fracn{(\vq^T\vn-c)}{r}$ and $t = \vn^T\at$.
We now minimize this expression over
$\mu,\sigma \geq 0$.
Setting the derivatives of
$\mathcal{L}$ w.r.t. $\mu$, and $\sigma$
equal to $0$ yields two equations to solve
for $\mu$ and $\sigma$:
$\|\at\|_2^2+\sigma^2-4\mu^2 r^2 = 2\sigma t $
and $\sigma = 2\mu r \fr+t$.
There are two cases: (A) If $t\geq -\fr\|\at\|_2$, then
$    \sigma =t + \fr \sqrt{ \frac{\|\at\|_2^2-t^2}{1-\fr^2} }$
    \quad\textrm{and}\quad
$    \mu =\frac{1}{2r}\sqrt{\frac{\|\at\|_2^2-t^2}{1-\fr^2}}$;
and (B) If $t < -\fr\|\at\|_2$,
then $\sigma =0$ and $\mu =\fracn{\|\at\|_2}{(2r)}$.
Substitution of these expressions into \eqref{eq:DMmu} yields the result in Lemma \ref{lem:max-thTb}.
\end{proof}

\begin{proof}[Theorem \ref{thm:dometest}]
For the nn-lasso, we reject $\at$ if
$\mu_{\DM}(\at)=\vq^T\at +\maxu_1(\vn^{T}\at,\|\at\|_2) < 1$, i.e., if $\vq^{T}\at < 1-\maxu_1(\vn^{T}\at,\|\at\|_2) =\tu(\vn^{T}\at,\|\at\|_2)$.
For the lasso we reject $\at$ if
$\vq^T\at +\maxu_1(\vn^T\at,\|\at\|_2) <1$ and
$-\vq^T\at +\maxu_1(-\vn^T\at,\|\at\|_2)\} <1$, i.e., if
$\vq^T\at < 1-\maxu_1(\vn^T\at,\|\at\|_2) =\tu(\vn^T\at,\|\at\|_2)$
and
$\vq^T\at > -(1-\maxu_1(-\vn^T\at,\|\at\|_2)) =\tl(\vn^T\at,\|\at\|_2)$.
\end{proof}

\section{Proofs \S\ref{sec:refine}}\label{sec:prrefine}

We make use of the following lemma from \cite{FCSS2013}.

\begin{lemma}\label{lem:diamR}
If $\RR=S(\vq,r)\cap\{\vn^T\vt\leq c\}$ is nonempty, then
$$
\diam(\RR) =
\begin{cases}
2\sqrt{r^2-(\vn^T\vq-c)^2}, & \text{if } \vq \notin \RR;\\
2r, & \text{otherwise}.
\end{cases}
$$
\end{lemma}

\begin{proof} [Lemma \ref{lem:circumsphere}]
The assumption that $0<\fr_d \leq 1$ is
equivalent to $\vq \notin \DM$.
Hence, under this assumption, by Lemma \ref{lem:diamR}, and equations   \eqref{eq:frd} and \eqref{eq:rb},
the diameter of $\DM=\DM(\vq_{1},r_{1};\vn,c)$ is
$2\sqrt{r_1^2-(\vn^T\vq_1-c)^2}
=2 r_1 \sqrt{1-\fr_{d}^{2}}=2\rd$.
So the diameter of the circumsphere of $\DM$ must be at least $2\rd$.

To show that the sphere  $\SPH(\vq_{d},\rd)$ with center $\vq_d$ and radius $\rd$ is the circumsphere, we show that every point $\vp$ on the boundary of $\DM$ is contained in $\SPH(\vq_d,\rd)$.
We can write $\vp=\vq_d + \alpha \vv + \beta \vn$,
where $\vv$ is a unit norm vector in $\vn^{\perp}$
and $\alpha, \beta$ are scalars with $\beta\leq 0$.
We need to show that $\|\vp-\vq_d\|_2^2 =\alpha^2+\beta^2\leq \rd^2$.
Since $\vp$ is on the boundary of $\DM$, either
$\beta=0$ and $\alpha^2\leq \rd^2$, or $\beta<0$ and $\|\vp-\vq\|_2^2 =r^2$.
In the first case, $\|\vp-\vq_d\|_2^2
= \alpha^2 \leq \rd^2$.
In the second case,
$r^2 = \|\vp-\vq\|_2^2 =\|\vq_d-\vq +\alpha\vv+\beta\vn\|_2^2
= \|( -\fr_d r+\beta)\vn +\alpha \vv\|_2^2
= \fr_d^2 r^2 -2\fr_d r \beta +\alpha^2 +\beta^2$.
Hence
$ \alpha^2+\beta^2
= r^2(1-\fr_d^2) +2\beta \fr_d r < r^2(1-\fr_d^2) =\rd^2$.
\end{proof}

\section{Proofs \S\ref{sec:THT}}\label{sec:prTHT}
\begin{proof}[Lemma \ref{lem:dmax}]
We first solve \eqref{eq:optprobm} ($m=2$) with
$\|\at\|_2=1$ by solving the Lagrangian problem:
\begin{align}
\begin{split}\label{eq:optimization_formulation_dual}
&\max_{\mu, \sigma, \lambda \geq 0} \ \min_{\vt}\ \
\mathcal{L}(\vt, \mu, \sigma, \lambda)
	= -\vt^T\at \\
	&\quad + \mu [(\vt-\vq)^T(\vt-\vq)-r^2] +\sigma (\vn_1^T\vt-c_1)\\
	&\quad +\lambda(\vn_2^T\vt-c_2).
\end{split}
\end{align}
Solving $\partial\mathcal{L}/\partial \vt=0$
for $\vt$ and substitution into $\mu_{\RR}$ and
 \eqref{eq:optimization_formulation_dual} yields:
\begin{align}
\begin{split} \label{eq:theta_relation}
&\mu_{\RR}(\at) =\textstyle \at^T\vq+\frac{1}{(2\mu)} -\frac{\sigma \at^T\vn_1}{(2\mu)}
-\frac{\lambda \at^T\vn_2}{(2\mu)} ,\\
&\mathcal{L}(\mu, \sigma, \lambda)
= \textstyle -\vq^T\at-\frac{1}{4\mu}-\frac{\sigma^2}{4\mu}
-\frac{\lambda^2}{4\mu}-\mu r^2\\
&\textstyle \quad  +\sigma r \fr_1+\lambda r \psi_2+\frac{\sigma}{2\mu}t_1+\frac{\lambda}{2\mu}t_2-\frac{\lambda\sigma}{2\mu}\tau,
\end{split}
\end{align}
where
$\psi_1 = \fracn{(\vq^T\vn_1-c_1)}{r}$,
$\psi_2 = \fracn{(\vq^T\vn_2-c_2)}{r}$,
$t_1 = \vn_1^T\at$,
$t_2 = \vn_2^T\at$ and
$\tau = \vn_1^T\vn_2$.
Setting the derivatives of $\mathcal{L}$ w.r.t.
$\mu, \sigma$ and $\lambda$, respectively, to zero yields:
\begin{equation} \label{eq:parameter_relation}
\begin{split}
&1+\sigma^2+\lambda^2-4\mu^2 r^2 = 2\sigma t_1+2\lambda t_2 -2\lambda \sigma \tau\\
&\sigma = 2\mu r \fr_1+t_1-\lambda\tau,\quad
\lambda = 2\mu r \fr_2+t_2-\sigma \tau .
\end{split}
\end{equation}
(Case I) If $\lambda = 2\mu r \fr_2+t_2-\sigma \tau<0$, then set
$\lambda =0$. Substitution into \eqref{eq:parameter_relation} yields:
$\sigma = 2\mu r\fr_1+t_1$
and
$1+\sigma^2-4\mu^2r^2-2\sigma t_1=0$.
There are two subcases:\\
    (IA) If $t_1>-\fr_1$, then
   $\sigma =t_1+\fr_1\sqrt{\frac{1-t_1^2}{1-\fr_1^2}}$,
   $\mu =\frac{1}{2r}\sqrt{\frac{1-t_1^2}{1-\fr_1^2}}$ and
    $\lambda < 0$ $\Leftrightarrow$
    $(\fr_2-\fr_1\tau)\sqrt{\frac{1-t_1^2}{1-\fr_1^2}}+t_2-t_1\tau<0$.
(IB) If $t_1 \leq -\fr_1$, then
$\sigma =0$,
$\mu =\fracn{1}{(2r)}$
and
$\lambda < 0 \Leftrightarrow t_2<-\fr_2$.

\noindent
(Case II) Suppose $\lambda = 2\mu r \fr_2+t_2-\sigma \tau>0$.
Again there are two subcases:
(IIA) If $\sigma = 2\mu r \fr_1+t_1-\lambda\tau<0$, then set $\sigma=0$.
Substitution into \eqref{eq:parameter_relation} yields:
$\lambda = 2\mu r \fr_2+t_2$
and
$1+\lambda^2-4\mu^2r^2-2\lambda t_2=0$.
Solving gives,
\begin{equation*}
\textstyle
\lambda =t_2+\fr_2\sqrt{\frac{1-t_2^2}{1-\fr_2^2}}
\quad\textrm{and}\quad
\mu = \frac{1}{2r}\sqrt{\frac{1-t_2^2}{1-\fr_2^2}}
\end{equation*}
with
$\lambda >0$ $\Leftrightarrow$ $t_2>-\fr_2$
and
$\sigma <0$ $\Leftrightarrow$ $(\fr_1-\fr_2\tau)
	\sqrt{\frac{1-t_2^2}{1-\fr_2^2}}+t_1-t_2\tau<0$.
(IIB) If $\sigma = 2\mu r \fr_1+t_1-\lambda\tau>0$, then substituting
$\lambda = 2\mu r \fr_2+t_2-\sigma \tau$
and
$\sigma = 2\mu r \fr_1+t_1-\lambda\tau$
into \eqref{eq:parameter_relation} yields,
$(1-\tau^2)\sigma = 2\mu r (\fr_1-\fr_2\tau)+t_1-t_2\tau$
and
$(1-\tau^2)\sigma^2+2\sigma(t_2\tau-t_1)+4\mu^2r^2(\fr_2^2-1)+1-t_2^2=0$.
Solving these equations gives:
    $\mu 		= \frac{1}{2r}\Delta$,
    $\lambda	= \frac{\fr_2-\fr_1\tau}{1-\tau^2}\Delta
    		+\frac{t_2-t_1\tau}{1-\tau^2}>0$, and
    $\sigma	= \frac{\fr_1-\fr_2\tau}{1-\tau^2}\Delta
    		+ \frac{t_1-t_2\tau}{1-\tau^2}>0$,
where
$    \Delta = \sqrt{\frac{1+2t_1t_2\tau-t_1^2-t_2^2-\tau^2}{1+2\fr_1\fr_2\tau-\fr_1^2-\fr_2^2-\tau^2}}$.

Substituting the expressions for $\mu, \sigma$ and $\lambda$ under the various conditions into \eqref{eq:theta_relation} yields:

\medskip
\noindent
[(1)] $t_1<-\fr_1, t_2<-\fr_2$: \quad $\mu_{\ADM}(\at)=\vq^T\at+r$.
\medskip

\noindent
[(2)] $ t_2>-\fr_2, \frac{t_1-t_2\tau}{\sqrt{1-t_2^2}}<\frac{\fr_2\tau-\fr_1}{\sqrt{1-\fr_2^2}}$: \quad
    $$\mu_{\ADM}(\at)=\vq^T\at+r\sqrt{(1-t_2^2)(1-\fr_2^2)}-rt_2\fr_2 .$$

\noindent
[(3)] $t_1>-\fr_1, \frac{t_2-t_1\tau}{\sqrt{1-t_1^2}}<\frac{\fr_1\tau-\fr_2}{\sqrt{1-\fr_1^2}}$: \quad
$$    \mu_{\ADM}(\at)=\vq^T\at+r\sqrt{(1-t_1^2)(1-\fr_1^2)}-rt_1\fr_1 .$$

\medskip
\noindent
[(4)]
$\frac{(t_1-t_2\tau)}{\sqrt{1+2t_1t_2\tau-t_1^2-t_2^2-\tau^2}}
> \frac{(\fr_2\tau-\fr_1)}{\sqrt{1+2\fr_1\fr_2\tau-\fr_1^2-\fr_2^2-\tau^2}}$
and
$\frac{(t_2-t_1\tau)}{\sqrt{1+2t_1t_2\tau-t_1^2-t_2^2-\tau^2}}>\frac{(\fr_1\tau-\fr_2)}{\sqrt{1+2\fr_1\fr_2\tau-\fr_1^2-\fr_2^2-\tau^2}}$:
    \begin{align*}
    \begin{split}
    \mu_{\ADM}(\at) & =\vq^T\at-\frac{r}{1-\tau^2}((\fr_1-\fr_2\tau)t_1+(\fr_2-\fr_1\tau)t_2)\\
    &+\frac{r}{1-\tau^2}\sqrt{(1-\tau^2+2\fr_1\fr_2\tau-\fr_1^2-\fr_2^2)} \\
    &\times
    \sqrt{(1-\tau^2+2t_1t_2\tau-t_1^2-t_2^2)}.
    \end{split}
    \end{align*}
For general $\at$ we use
$\mu_{\ADM} (\at)
= \|\at\|_2 \mu_{\ADM}(\at/\|\at\|_2)$.
So in each of the above expressions we replace $\at$, $t_1=\vn_1^T\at$,
and $t_2=\vn_2^T\at$ by $\at/\|\at\|_2$, $t_1/\|\at\|_2$ and
$t_2/\|\at\|_2$, respectively. Then multiply each expression by $\|\at\|_2$.
This yields the result in Lemma \ref{lem:dmax}.
\end{proof}

\begin{proof}[Theorem \ref{thm:THT}]
This is almost identical to the proof of Theorem \ref{thm:dometest}
and is hence omitted.
\end{proof}

\section{Proofs \S\ref{sec:comp}}\label{sec:prcomp}

\begin{proof}[Lemma \ref{lem:conjt}]
Note
$\max\{\mu_{\RR}(\vb_{i}),\mu_{\RR}(-\vb_{i})\}
= \max_{\vt\in\RR} \max\{\vt^{T}\vb,-\vt^{T}\vb\}
= \max_{\vt\in\RR} |\vt^{T}\vb|$.
If $\RR_{1}$ or $\RR_{2}$ is empty, the result is clear. Hence assume each is nonempty.
For the lasso, if $\vb_i$ is rejected by
$T_{\RR_1} \disjct T_{\RR_2}$,
then either $\max_{\vt\in\RR_1} |\vt^T\vb_i| < 1 $ or $\max_{\vt\in\RR_2} |\vt^T\vb_i| < 1 $.
Without loss of generality assume
$\max_{\vt\in\RR_1} |\vt^T\vb_i| < 1 $.
Since $\RR_1\cap \RR_2$ is a subset of $\RR_1$,
this implies that $\max_{\vt\in\RR_1\cap\RR_2} |\vt^T\vb_i| \leq \max_{\vt\in\RR_1} |\vt^T\vb_i| < 1$,
so $\vb_i$ is also rejected by $T_{\RR_1 \cap \RR_2}$. Therefore
$T_{\RR_1} \disjct T_{\RR_2} \comp T_{\RR_1 \cap \RR_2} $.
The proof for the nn-lasso is similar.
\end{proof}

\section*{Acknowledgment}
This work partially supported by NSF grant CIF 1116208.

\bibliographystyle{IEEEtran}
{\small
\bibliography{Screening2013}
}

\begin{IEEEbiography}{Zhen James Xiang}
received the B.E. degree B.E. in 2007 from  Department of Electrical Engineering, Tsinghua University, China, graduating with honors and GPA rank 1/164. He received the M.A. degree and the Ph.D. degree in Electrical Engineering from Princeton University in 2009 and 2012 respectively.  He is currently a quantitative researcher at Citadel LLC in Chicago.
He has received several awards and honorable mentions for his scholarship, including:
Best Student Paper Honorable Mention Award, NIPS (2011); Charlotte Elizabeth Procter Honorific Fellowship of Princeton University (2011-2012);
Qualcomm Innovation Fellowship Finalist (2011);
Francis Robin Upton Fellowship, Princeton University (2007-2011);
Distinguished Graduate Award of Beijing City (2007);
Distinguished Graduate of Tsinghua University (2007);
and in 2003, a Gold Medal at the International Mathematics Olympiad, Tokyo, Japan, where he ranked 12th among 457 participants from 84 countries.
\end{IEEEbiography}
\begin{IEEEbiography}{Yun Wang}
received the B.S. degree in Electrical Engineering with highest honors from Shanghai Jiao Tong University in 2011 and the M.A. and Ph.D. degrees in Electrical Engineering from Princeton University in 2013 and 2015, respectively.  His doctoral research focused on machine learning, optimization and statistical signal processing. He joined Amazon as a machine learning scientist in the fall of 2015. His honors include the Distinguished Graduate Award of the City of Shanghai (2011) and the Anthony Ephremides Fellowship in Electrical Engineering, Princeton University (2011).
\end{IEEEbiography}
\begin{IEEEbiography}{Peter J. Ramadge}
received the B.Sc., B.E. and the M.E. degree from the University of Newcastle, Australia, and the Ph.D. degree in Electrical Engineering from the University of Toronto, Canada. He joined the faculty of Princeton University in September 1984, where he is currently Gordon Y. S. Wu Professor of Engineering, and Professor of Electrical Engineering.
He is a Fellow of the IEEE and a member of SIAM. He has received several honors and awards including: a paper selected for inclusion in IEEE book ÒControl Theory: Twenty Five Seminal Papers (1932-1981)Ó; an Outstanding Paper Award from the Control Systems Society of the IEEE; a listing in ISIHighlyCited.com; the Convocation Medal for Professional Excellence from the University of Newcastle, Australia;
an IBM Faculty Development Award;
and the University Medal from the University of Newcastle, Australia.
His current research interests are in statistical signal processing, machine learning and various applications, including: data analysis, classification, prediction, medical and fMRI data analysis, and video and image processing.
\end{IEEEbiography}

\vfill
\end{document}